\documentclass[11pt]{article}
\pdfoutput=1
\usepackage{fullpage}
\usepackage[margin=1in]{geometry}
\usepackage[utf8]{inputenc}

\usepackage{graphicx}

\usepackage{amsmath}
\usepackage{amssymb}
\usepackage{amsthm}

\usepackage{thmtools}
\usepackage{thm-restate}
\usepackage{bbm}
\usepackage{bm}
\usepackage{verbatim}
\usepackage[pagebackref]{hyperref}
\hypersetup{colorlinks=true,citecolor=blue,linkcolor=blue}
\usepackage{array}
\usepackage{caption}
\usepackage{subcaption}
\usepackage{float}
\usepackage{changepage}
\usepackage{enumitem}
\usepackage{cases}
\usepackage[dvipsnames]{xcolor}
\usepackage{algorithm}
\usepackage{algpseudocode}
\usepackage{cleveref}

\usepackage{xfrac}
\usepackage{tikz, pgfplots}

\usetikzlibrary{decorations.pathreplacing}

\usepackage[authoryear]{natbib}

\newcommand{\R}{\mathbb{R}}

\renewcommand{\hat}{\widehat}
\renewcommand{\tilde}{\widetilde}

\DeclareMathOperator*{\argmax}{arg\,max}

\DeclareMathOperator{\E}{\mathbb{E}}

\newcommand{\diff}{\text{d}}

\theoremstyle{definition}
\newtheorem{theorem}{Theorem}[section]
\newtheorem*{theorem*}{Theorem}
\newtheorem{lemma}[theorem]{Lemma}
\newtheorem{definition}[theorem]{Definition}

\newtheorem{proposition}[theorem]{Proposition}

\newtheorem{remark}[theorem]{Remark}
\newtheorem{claim}[theorem]{Claim}
\newtheorem{example}[theorem]{Example}
\newtheorem*{example*}{Example}

\newcommand{\cF}{{\mathcal F}}
\newcommand{\cA}{{\mathcal A}}

\newcommand{\act}{a}
\newcommand{\bestact}{a^*}
\newcommand{\vact}{\bm{\act}}
\newcommand{\actsp}{A}

\newcommand{\util}{u}

\newcommand{\decisionu}{U}

\newcommand{\state}{\theta}
\newcommand{\statesp}{\Theta}

\newcommand{\prob}{p}

\newcommand{\expect}[2]{{\mathbf{E}}_{#1}\left[#2\right]}

\newcommand{\score}{S}

\newcommand{\robustcal}{\textsc{CDL}}

\newcommand{\lcal}[1]{K_{#1}}

\newcommand{\ind}[1]{\mathbb{I}\left[#1\right]}

\newcommand{\swapS}{\predswap}

\newcommand{\pred}{\prob}
\newcommand{\vpred}{\bm{\pred}}
\newcommand{\vstate}{\bm{\state}}

\newcommand{\empp}{\hat{\pred}}
\newcommand{\vempp}{\bm{\empp}}

\newcommand{\swapV}{\textsc{VCDL}}

\newcommand{\breg}{\textsc{Breg}}
\newcommand{\vbreg}{\textsc{VBreg}}

\newcommand{\reals}{\mathbb{R}}

\newcommand{\actswap}{\textsc{Swap}}

\newcommand{\predswap}{\textsc{PSwap}}

\newcommand{\q}{q}

\newcommand{\empq}{\hat{\q}}
\newcommand{\qcount}{n}

\newcommand{\kink}{\mu}

\newcommand{\G}{G}

\newcommand{\expert}{l}

\newcommand{\w}{w}

\newcommand{\vw}{\bm{\w}}

\newcommand{\msmwc}{\mathcal{A}}

\newcommand{\s}{s}

\newcommand{\asympO}{\Tilde{O}}

\newcommand{\regret}{R}

\newcommand{\prednoround}{\Tilde{\pred}}
\newcommand{\vprednoround}{\bm{\prednoround}}

\newcommand{\LL}{\mathcal{D}}

\newcommand{\distcal}{\textsc{distCal}}

\newcommand{\smooth}{\textsc{smCal}}

\newcommand{\external}{\textsc{Ext}}

\newcommand{\ucal}{\textsc{UCal}}

\newcommand{\sps}[1]{^{(#1)}}

\newcommand{\ece}{\textsc{ECE}}

\newcommand{\cfdl}{\textsc{CFDL}}

\newcommand{\lunjia}[1]{{\color{red}  [\text{Lunjia:} #1]}}

\title{Calibration Error for Decision Making\footnote{A previous version of the paper was titled ``Predict to Minimize Swap Regret for All Payoff-Bounded Tasks''. We are grateful to Jason Hartline for his helpful comments and for suggesting the name of our Calibration Decision Loss. We also thank Mingda Qiao for helpful discussions about the online calibration literature.}}
\author{Lunjia Hu\thanks{Work done while LH was a PhD student at Stanford, supported by Moses Charikar’s and Omer Reingold’s Simons Investigators awards and the Simons Foundation Collaboration on the Theory of Algorithmic Fairness.}\\
Harvard University\\
\texttt{lunjia@g.harvard.edu}
\and Yifan Wu\thanks{Supported by NSF CCF-2229162. }\\
Northwestern University\\
\texttt{yifan.wu@u.northwestern.edu}}
\date{}

\allowdisplaybreaks
\begin{document}

\maketitle

\begin{abstract}
Calibration allows predictions to be reliably interpreted as probabilities by decision makers. 
We propose a decision-theoretic calibration error, the Calibration Decision Loss (\robustcal), defined as the maximum improvement in decision payoff obtained by calibrating the predictions, where the maximum is over all payoff-bounded decision tasks. Vanishing $\robustcal$ guarantees the payoff loss from miscalibration vanishes simultaneously for all downstream decision tasks. We show separations between $\robustcal$ and existing calibration error metrics, including the most well-studied metric Expected Calibration Error ($\ece$). Our main technical contribution is a new efficient algorithm for online calibration that achieves near-optimal $O(\frac{\log T}{\sqrt{T}})$ expected $\robustcal$, bypassing the $\Omega(T^{-0.472})$ lower bound for $\ece$ by \citet{sidestep}. 

\end{abstract}

\section{Introduction}

Given a sequence of predictions indicating, say, the chances of rain each day for a period of $T$ days, an intuitive way to assess the quality of these predictions is to check for calibration: for example, among the days predicted to have a 60\% chance of rain, is the fraction of rainy days indeed 60\%?
Formally, suppose a prediction of $\q_i\in [0,1]$ is received on $\qcount_i$ days, and among the $\qcount_i$ days, $m_i$ days are actually rainy.
Calibration, a notion that originated from the forecasting literature \citep{dawid}, requires the relationship 
$m_i=\q_i \qcount_i$
to hold for every prediction value $\q_i$. 
The most well-studied calibration error metric is the Expected Calibration Error (\ece), defined as $\frac{1}{T}\sum_i|m_i - q_in_i|$, or equivalently written as the average prediction bias $\frac{1}{T} \sum_{i}n_i|q_i - \frac{m_i}{n_i}|$, 
where $|q_i - \frac{m_i}{n_i}|$ is the absolute bias when $q_i$ is predicted. This gives a well-defined calibration error metric: a predictor is perfectly calibrated if and only if its $\ece$ is zero.

In this work, we focus on the value of calibration for decision making. 
In the economics literature, the value of predictions is quantified by the payoff from decision making \citep{blackwell1951comparison}. Consider a decision maker who needs to choose an action $\act\in \actsp$ 
to maximize their decision payoff $\decisionu(\act,\state)$, an arbitrary function of the action $\act$
(e.g.\ take an umbrella or not) and the true binary state $\state\in \{0,1\}$ 
(e.g.\ rainy or not). 
When the true state $\theta$ 
is unknown and only a prediction $p\in [0,1]$ 
is given, the decision maker can trust the prediction and take the action $a$ that maximizes the expected payoff $\E_{\state\sim \pred}[\decisionu(\act,\state)]$
, where ``trust'' means the decision maker \emph{assumes} that the distribution of the state  
is indeed as predicted ($\Pr[\state = 1] = p$). 
Thus, the prediction manifests its value in the payoff received by the decision maker who best responds to the prediction as if the predictions were correct.


In the following examples, we explain how calibrated predictions are \emph{trustworthy} for all downstream decision-makers, regardless of their specific decision tasks. \Cref{example:intro-miscalibrated} shows a miscalibrated predictor and \Cref{example:intro-calibrated} shows a calibrated predictor.
\begin{example}
\label{example:intro-miscalibrated}
Among the $T$ samples of (prediction, state), $\frac{T}{2}$ predictions are $0.4$ and the other half are $0.6$. When the prediction is $0.4$, the actual empirical frequency of state $1$ is $0.2$; when the prediction is $0.6$, the actual empirical frequency of state $1$ is $0.8$. 
\begin{center}
    \begin{tabular}{r|ccccccccccc}
 Prediction $\pred_t$ & 0.4 & 0.6 & 0.4 & 0.4 & 0.6 & 0.4 & 0.6 & 0.6 & 0.4 & 0.6 &  \dots\\
 State $\state_t$   & 0 & 1 & 1 & 0 & 1 & 0 & 1 & 0 & 0 & 1 & \dots
\end{tabular}
\end{center}
 \end{example}
 \noindent After observing enough samples of (prediction, state), the decision maker will figure out the ``meaning'' of each prediction: $0.4$ actually means a $0.2$ probability of state being $1$, and $0.6$ actually means $0.8$. 
Based on this observation, the decision maker will no longer \emph{trust} the original prediction values when making decisions. 
Instead, if we calibrate the predictions by changing them to the actual empirical frequencies (e.g.\ $0.4\to 0.2$, and $0.6\to 0.8$, see \Cref{example:intro-calibrated}), the decision maker 
 can now trust the predictions by interpreting them as the true probabilities when making decisions. 
 \begin{example}
     \label{example:intro-calibrated}
     Among the $T$ samples of (prediction, state), $\frac{T}{2}$ predictions are $0.2$ and the other half are $0.8$. When the prediction is $0.2$, the actual empirical frequency of state $1$ is indeed $0.2$; when the prediction is $0.8$, the actual empirical frequency of state $1$ is indeed $0.8$. 
\begin{center}
    \begin{tabular}{r|ccccccccccc}
 Prediction $\pred_t$ & 0.2 & 0.8 & 0.2 & 0.2 & 0.8 & 0.2 & 0.8 & 0.8 & 0.2 & 0.8 &  \dots\\
 State $\state_t$   & 0 & 1 & 1 & 0 & 1 & 0 & 1 & 0 & 0 & 1 & \dots
\end{tabular}
\end{center}
 \end{example}

\subsection{The Decision Loss from Miscalibration}

The examples above seem to suggest that calibration is a \emph{valuable} property of a predictor for downstream decision makers. Is there a way to precisely \emph{quantify} this value? Does the widely used calibration error $\ece$ give the right quantification? Our main result reveals a fundamental gap between $\ece$ and the value of calibration for decision making. In fact, $\ece$ overestimates the decision loss from miscalibration. 
 
To provide some intuition for this gap, let us
consider the miscalibrated predictor in \Cref{example:intro-miscalibrated} and a decision task defined as follows, where we show that $\ece$ does not give the payoff loss due to miscalibration.

The  decision task has binary states $\statesp = \{0, 1\}$ (not rainy or rainy) and actions $A = \{0, 1\}$, where $1$ stands for taking an umbrella and $0$ for not. The payoff is $1$ if the decision maker's action matches the state and $0$ otherwise. 
Best responding to the prediction, the decision maker chooses to bring an umbrella when the predicted chance of rain is at least half ($p_t\geq 0.5$). The example is visually explained in \Cref{fig:intro-example}. 
\begin{figure}[htbp]
    \centering
    \begin{tikzpicture}[scale = 0.5]
    \draw[thick] (0,0) -- (8,0);

    \foreach \x in {0, 1,  3, 4, 5, 7, 8} {
        \draw[thick] (\x,0.2) -- (\x,-0.2);
    }
    
    \node[below] at (1, -0.3){\tiny{0.2}};
    \node[below] at (3, -0.3){\textcolor{blue}{\textbf{{\tiny{0.4}}}}};
    \node[below] at (5, -0.3){\textcolor{blue}{\textbf{\tiny{0.6}}}};
    \node[below] at (7, -0.3){\tiny{0.8}};

    \node[above] at (-2, -0.3){prediction};
    \node[below] at (0, -0.3){$0$};
    \node[below] at (8, -0.3){$1$};
    \draw[red, thick] (4,0.5) -- (4,-0.5);
    
    \node[above] at (4, 0.5) {$\frac{1}{2}$};
    \node[below] at (4, -0.7){(threshold)};
\end{tikzpicture}
    \caption{In this example, $\ece$ overestimates the decision loss from miscalibration for a specific decision task. The plot visualizes the predictions in $[0, 1]$. The best-response decision rule changes action at threshold $\sfrac{1}{2}$ (red).  
    When the miscalibrated predictor predicts $0.4$ (blue), the actual empirical frequency is $0.2$; and when $0.6$ (blue) is predicted, the empirical frequency is $0.8$. Miscalibration induces no loss to the decision maker, since in both cases the prediction and the corresponding empirical frequency lie on the same side of the threshold, recommending the same action. }
    \label{fig:intro-example}
\end{figure}
The average prediction bias, $\ece$, of the (miscalibrated) predictor in \Cref{example:intro-miscalibrated} is $0.2$. However, the decision maker's payoff loss from this miscalibration is actually $0$. 
This is because the decision maker would take the same actions even if the predictions are recalibrated as in \Cref{example:intro-calibrated}. Indeed, a prediction $0.4$ and the recalibrated version $0.2$ are both below the decision threshold of $0.5$, leading to action $0$. Similarly, a prediction $0.6$ and the recalibrated version $0.8$ both lead to action $1$. 


In this example, for the particular decision maker, the loss caused by miscalibration is zero, i.e.\ recalibrating the predictions provides zero value. In contrast, $\ece$ is $0.2 > 0$, overestimating the decision loss from miscalibration. While this example only considers a specific decision task with two actions, we ask the natural question for all decision tasks with arbitrarily many actions - does $\ece$ simultaneously and significantly overestimate the decision loss for \emph{every} downstream task?


To answer this question, we propose a calibration error for decision making, the Calibration Decision Loss ($\robustcal$), which quantifies the worst-case payoff loss caused by miscalibration of a predictor. Our main result demonstrates a fundamental separation between $\robustcal$ and existing error metrics including $\ece$. Inspired by this separation, we give a new efficient  algorithm for online prediction that minimizes $\robustcal$ at a faster rate than what is possible for $\ece$ (see \Cref{sec:intro-online} for more details).

Our definition of $\robustcal$ can be decomposed into two steps. First, for a fixed decision task, miscalibration incurs a loss to the decision maker, which we call Calibration Fixed Decision Loss (CFDL). 
Consider a decision maker who best responds to a miscalibrated predictor as in \Cref{example:intro-miscalibrated}. We define CFDL as the payoff increase (averaged over $T$ rounds) that the decision maker could have achieved if we had calibrated the predictions as in \Cref{example:intro-calibrated}.
Second, since perfect calibration guarantees trustworthiness for all decision makers, we measure the calibration error by defining CDL as the worst-case (i.e., maximum) CFDL over all decision tasks with payoffs bounded in $[0, 1]$. 
Here we restrict the payoffs to be bounded solely for the purpose of normalization: the CFDL of a decision task scales proportionally if we multiply the payoffs by any positive constant, and it remains the same if any constant is added to the payoffs. Beyond that, we make no additional restrictions on the decision tasks. In particular, we allow each decision task to have an arbitrarily large action space $A$, and the CDL is the maximum CFDL over all such tasks.
Thus, vanishing CDL implies that the payoff loss from miscalibration vanishes for all decision tasks.

Similar to ECE, our CDL is a well-defined calibration error metric, with the basic property that a sequence of predictions is perfectly calibrated if and only if its CDL is zero, as implied by \cite{foster1997calibrated}.%
\footnote{In the binary decision task in \Cref{fig:intro-example}, the $\cfdl$ of the miscalibrated predictor (\Cref{example:intro-miscalibrated}) is $0$. However, other decision tasks exist where the miscalibrated predictor has positive $\cfdl$, leading to a non-zero $\robustcal$.}
However, we show that CDL is different from, and is not even a constant factor approximation of existing calibration error metrics, including $\ece$, the $\lcal{2}$  calibration error (i.e.\ the expected squared prediction bias), the smooth calibration error, and the distance to calibration \citep{utc}.
%
%
%
In particular, we observe that ECE is often a significant overestimation of CDL.
As a strong demonstration of this observation, we give an efficient algorithm for online prediction that minimizes CDL at a near-optimal rate of $O(T^{-1/2}\log T)$, which is faster than what is possible for ECE: it surpasses the $\Omega(T^{-0.472})$ lower bound for ECE by \citet{sidestep}. We discuss this result with more details in \Cref{sec:intro-online} below.

\subsection{Online  Calibration}
\label{sec:intro-online}
Let us now consider the algorithmic task of making sequential predictions for $T$ rounds. The goal is to achieve asymptotic calibration, meaning that the predictions are arbitrarily ``close'' to perfect calibration for sufficiently large $T$, with ``closeness'' measured by some calibration error. Though it may seem impossible, 
a classic and remarkable result by \citet{foster1998asymptotic} shows that asymptotic calibration can be achieved  
without \emph{any} knowledge of what the state will be in each round. 

Concretely, \citet{foster1998asymptotic} consider the following \emph{online binary prediction} setting. The algorithm (i.e., predictor) interacts with an adversary for $T$ rounds. In each round $t = 1,\ldots,T$, the predictor makes a prediction $\pred_t\in [0,1]$ and the adversary reveals the true state $\state_t\in \{0,1\}$. 
The only knowledge that the predictor can use to produce $p_t$ is the history $(p_1,\theta_1,\ldots,p_{t-1},\theta_{t-1})$, which may contain zero information about the new state $\theta_t$. The adversary, on the other hand, can choose $\theta_t$ arbitrarily, based on the history \emph{and the predictor's strategy}, so as to maximize the calibration error of the predictions. 

\citet{foster1998asymptotic} focus on $\ece$ as the metric for calibration error. They prove the existence of a randomized algorithm that guarantees $O(T^{-1/3})$ expected ECE, where the expectation is over the randomness of the algorithm. 
Their use of a randomized algorithm breaks the impossibility of asymptotic calibration if a deterministic prediction strategy were used.
To see this impossibility, if the algorithm predicts deterministically, the adversary can pick $\theta_t = 0$ whenever the prediction $p_t$ is above $0.5$, and $\theta_t = 1$ otherwise, yielding $\ece\ge 1/2$. In contrast, the predictor in \cite{foster1998asymptotic} draws $p_t$ randomly from a distribution. The adversary, knowing the predictor's strategy, can choose $\theta_t$ based on this distribution, but it can no longer choose $\theta_t$ based on the realized value of $p_t$ drawn from the distribution.

This remarkable first result of \cite{foster1998asymptotic} opened up the literature of online asymptotic calibration, with the focus of error metric mainly on $\ece$.  
On the upperbound side, subsequent work constructs polynomial-time algorithms for $\ece$ minimization \citep[e.g.][]{foster99,foster2021hedging}. On the lowerbound side, a recent breakthrough of \citet{sidestep} proves an $\Omega(T^{-0.472})$ lower bound for ECE, improving over the long-known natural $\Omega(T^{-1/2})$ bound.%
\footnote{Subsequent to our work, \citet{dagan2024improved} show there exists a sequential prediction strategy that achieves $O(T^{-\sfrac{1}{3}-\varepsilon})$ expected $\ece$ for some $\varepsilon > 0$. They also give an improved lowerbound for $\ece$ over the $\Omega(T^{-0.472})$ in \citet{sidestep}.}


It is unclear, however, that for decision making, ECE is the calibration error metric we wish to optimize.
The economic value of making good predictions lies in helping  downstream decision makers achieve better payoff. 
If the goal of calibration is to generate trustworthy predictions that induce no payoff loss to all decision makers, what we should really minimize is  $\robustcal$, rather than ECE. 


Prior to our work, nothing was known about $\robustcal$ minimization beyond what is implied by minimizing $\ece$. Specifically, the relationship $\robustcal \le 2\,\ece$ can be inferred from \citet{kleinberg2023u}. 
Therefore, any ECE minimization algorithm also guarantees that $\robustcal$ vanishes at the same rate (i.e., up to a factor of $2$). 
The interesting question is to go beyond this implication: can we achieve a better error rate for CDL than that is possible for ECE, surpassing the $\Omega(T^{-0.472})$ lower bound?

Our work gives a positive answer.
We give an efficient randomized algorithm that guarantees $O(T^{-1/2}\log T)$ expected $\robustcal$ (\Cref{thm:sqrt MSR}).
This error rate surpasses the $\Omega(T^{-0.472})$ lowerbound for $\ece$ and is optimal up to a logarithmic factor: there is a natural $\Omega(T^{-1/2})$ lower bound for CDL minimization from drawing each state $\theta_t$ independently and uniformly from $\{0,1\}$.\footnote{Consider the decision task with two actions ($A = \{0,1\}$) and the 0-1 payoff function $U(a,\theta) = \ind{a = \theta}$. When the states are independent and unbiased coin flips, the total payoff drops below $T/2 - \Omega(\sqrt T)$ with constant probability, inducing $\Omega(T^{-1/2})$ CFDL.} By definition, our algorithm guarantees that for every downstream decision task, the payoff loss from miscalibration ($\cfdl$) vanishes at the same near-optimal rate simultaneously.

Our work establishes a comprehensive understanding of the optimal rate for CDL, compared to the significant gaps in the current best upper and lower bounds for the other two recently popular metrics in online asymptotic calibration (see \Cref{sec:related} for more details):
\begin{itemize}
    \item $\ece$: $O(T^{-1/3})$ and $\Omega(T^{-0.472})$ \citep{foster1998asymptotic,sidestep};
    \item Distance to calibration \citep{utc}: $O(T^{-1/2})$ and $\Omega(T^{-2/3})$ \citep{qiao-distance, elementary}.
\end{itemize}

Certainly, our result would be impossible without a  separation between $\robustcal$ and $\ece$. This separation has been hinted in our examples earlier, and we will give a more in-depth explanation in \Cref{sec:tech-overview} with an overview of the techniques in this paper.

Our result implies a substantial strengthening of an independent result by \citet{roth2024forecasting} on \textit{swap regret} minimization. Swap regret minimization has been studied extensively in the literature of online learning \citep{hart2000simple, blum-mansour}. Consider a fixed decision task repeated for $T$ rounds. The swap regret is the payoff increase (averaged over $T$ rounds) when the decision maker is allowed to apply a mapping $\sigma:\actsp\to\actsp$ and swap each action to another action in hindsight. The swap regret is stronger than the more commonly studied external regret, which only allows the decision maker to swap each action to a fixed action in hindsight.  \citet{foster1999regret} have observed that a predictor is perfectly calibrated if and only if it guarantees no swap regret for every downstream decision maker who trusts the prediction. 
\citet{roth2024forecasting} propose an algorithm that achieves $O\Big(|\actsp|\sqrt{\frac{\log T}{T}}\Big)$ swap regret for every downstream decision maker, where $|\actsp|$ is the number of actions. 

In our paper, our algorithm guarantees the same near-optimal $O(\frac{\log T}{\sqrt{T}})$ swap regret for every downstream decision maker simultaneously, removing the dependence on the number $|\actsp|$ of actions in \citet{roth2024forecasting}. 
This  directly follows from the fact that swap regret is upperbounded by $\cfdl$ for a fixed a decision task. To see this, first recall that $\cfdl$ is the improvement in payoff when predictions are calibrated to the conditional empirical frequencies. While the calibrated predictions may suggest different actions to the decision maker, this improvement can be equivalently written as the regret when, in hindsight, the decision maker is allowed to swap actions whenever they receive a different prediction. 
Thus, $\cfdl$ is stronger since the modification rule is finer-grained than that of the swap regret. 
If two predictions suggest the same best-response action to a decision maker, the modification rule in $\cfdl$ allows the decision maker to swap the actions differently, while the swap regret does not.


The \textit{U-calibration} error $\ucal$, introduced by \citet{kleinberg2023u}, is closely related to our $\robustcal$, despite being qualitatively weaker (i.e.\ $\ucal$ being zero is necessary but insufficient for calibration). 
Both errors are defined as the maximum payoff increase over decision tasks:   
CDL is the maximum payoff increase from calibrating the predictions; whereas U-calibration is the maximum payoff increase (a.k.a.\ \emph{external regret}) from changing every prediction $p_t$ to the best fixed prediction (i.e., predicting the overall average $(\theta_1 + \cdots + \theta_T)/T$ every round).%
\footnote{For a fixed decision task, the external regret may be negative, but the U-calibration error is non-negative because it is the maximum over all payoff-bounded decision tasks (including the degenerate task where the payoff is always zero).} Therefore, the U-calibration error lowerbounds CDL. 
To see why $\ucal$ is not a well-defined calibration error, consider the miscalibrated predictor such that $(p_t,\theta_t) = (1/4, 0)$ for $T/2$ rounds, and $(p_t,\theta_t) = (3/4,1)$ for the remaining $T/2$ rounds. The miscalibrated predictor always induces better or equal payoff compared to the best fixed prediction $1/2$, giving zero U-calibration error. See a detailed discussion in \Cref{example:u calib}.

While our result appears very similar to the result of \citet{kleinberg2023u}, we use substantially different techniques since our $\robustcal$ is qualitatively stronger. 
\citet{kleinberg2023u} give a randomized algorithm that guarantees an optimal $O(T^{-1/2})$ expected U-calibration error. 
Their prediction algorithm 
ensures that each best-responding decision maker takes actions as if they are locally running the classic Hedge algorithm \citep[see][]{mw}, 
which is specific to the external regret. Thus, their guarantee does not directly extend to our CFDL. 

\vspace{3mm}

\subsection{Results Overview}

We discuss the connections between   $\robustcal$ and other calibration error metrics in \Cref{sec:intro-calibration error metric}. We describe our technical ideas behind our main result for online calibration in \Cref{sec:tech-overview}. 

\subsubsection{Properties of  $\robustcal$ }

\label{sec:intro-calibration error metric}

We propose $\robustcal$ as a calibration error metric and study its properties. In \Cref{sec:computation}, we show, for general non-binary state space, $\robustcal$  can be computed in time polynomial in the size of the prediction space $|Q|$ and the state space $|\statesp|$ by solving a linear program. 

In \Cref{sec:connection}, we discuss the separation between $\robustcal$ and other calibration error metrics, including  $\ece$, the $\lcal{2}$ calibration error, and the smooth calibration error. We consider the average calibration error calculated on $T$ samples of the prediction and the state. Our results show previous results on minimizing $\lcal{2}$ and $\smooth$ do not directly apply. We also show the separation between $\robustcal$ and the U-calibration error $\ucal$ from \cite{kleinberg2023u}, where $\ucal = 0$ is necessary but not sufficient for the predictions to be calibrated. 
\begin{itemize}
    \item $\ece$ is polynomially related\footnote{If two error metrics $A$ and $B$ are polynomially related, we can find two polynomial functions of error $A$ as upperbound and lowerbound of error $B$, respectively. } to $\robustcal$. 
    \begin{align*}
        \ece^2\leq &\robustcal\leq 2\ece
\end{align*}
We give examples where inequalities are asymptotically tight. In fact, the lower bound is attained by the same  example in \Cref{sec:tech-overview}.
\item $\lcal{2}$, defined as average squared bias,  is polynomially related to $\robustcal$. 
\begin{align*}
        \lcal{2}\leq &\robustcal\leq 2\sqrt{\lcal{2}}.
\end{align*}
We give examples where inequalities are asymptotically tight. There exists an online algorithm that achieves $\asympO(\frac{1}{\sqrt{T}})$ $\lcal{2}$ calibration error \citep{rothbook}. 
\item The smooth calibration error and the distance to calibration are not polynomially related to $\robustcal$, where we give examples. 
\item The U-calibration error lowerbounds $\robustcal$, but is not polynomially related. 
\end{itemize}


\subsubsection{Online $\robustcal$ Minimization: Technical Overview}
\label{sec:tech-overview}
In this section, we give an overview of our techniques for achieving near-optimal $\robustcal$. The key idea behind our $O(T^{-1/2}\log T)$ guarantee for $\robustcal$ comes from the observation that the $\robustcal$ can often be significantly smaller than $\ece$, despite their linear relationship in the worst-case. This allows us to bypass the $\Omega(T^{-0.472})$ lower bound for $\ece$. Here we provide a typical example where the $\robustcal$ is much smaller than $\ece$. We will first describe the samples from a miscalibrated predictor, then show $\ece$ is higher than $\robustcal$ by calculation. Based on the intuition from this example, we establish a general lemma (\Cref{lm:attribute-informal}) which plays a crucial role in our analysis.

\paragraph{Miscalibrated Predictor.} We now describe the observed samples from a miscalibrated predictor. The predictions are discretized to a finite set $Q:=\{q_1,\ldots,q_m\}\subseteq[0,1]$, where $q_i = i/m$ for $i = 1,\ldots,m$. 
We view each $q_i$ as a ``bucket'', and thus our predictor makes predictions that fall into these buckets.
We use $\ind{\cdot}$ to denote the 0-1 indicator function: $\ind{\text{statement}} = 1$ if the statement is true, and $\ind{\text{statement}} = 0$ if the statement is false.

For a sequence  of $T$ predictions $\vpred = (p_1,\ldots,p_T)\in Q^T$ made by our predictor and the corresponding true states $\vstate = (\theta_1,\ldots,\theta_T)\in \{0,1\}^T$, let $n_i$ denote the number of predictions in bucket $i$:
\begin{equation}
\label{eq:ni-intro}
n_i:= \sum_{t = 1}^T \ind{p_t = q_i},
\end{equation}
and let $\hat q_i$ denote the empirical average of the true outcomes corresponding to the $n_i$ predictions:
\begin{equation}
\label{eq:hat-q-intro}
\hat q_i := \frac 1{n_i}\sum_{t=1}^T \theta_t \ind{p_t = q_i}. 
\end{equation}

Assuming $\sqrt T$ is an integer, let us choose $m = \sqrt T$. Assume for simplicity that each bucket contains the same number of predictions: $n_i = \sqrt T$ for every $i = 1,\ldots, m$.
A typical guarantee one can often obtain using existing online learning techniques is the following bound on the deviation between $\hat q_i$ and $q_i$:
\[
|\hat q_i - q_i| \lesssim 1/\sqrt {n_i} = T^{-1/4}.
\]
Thus, in this example, we construct $|\hat q_i - q_i| = T^{-1/4}$ for $i = 1,\ldots, m$ for simplicity. 

\paragraph{$\ece$.} $\ece$ of the example above can be computed as
\begin{equation}
\label{eq:K1-2}
\ece(\vpred,\vstate) = \frac{1}{T}\sum_{i=1}^m n_i |q_i - \hat q_i| = T^{-1/4}.
\end{equation}
This is even worse than the $\ece = O(T^{-1/3})$  guarantee by the algorithm of \cite{foster1998asymptotic}, mainly because here we discretize the prediction space into $m = \sqrt T$ buckets, but the optimal choice would be $m \approx T^{1/3}$.

\paragraph{$\cfdl$ and $\robustcal$.} However, the $\robustcal$ in this example is much smaller: in fact we have $\robustcal \le O(T^{-1/2})$. To prove this fact, let us first consider a specific decision task with action space $A = \{0,1\}$, where the payoff is given by
\[
U(a,\theta) = \ind{\theta = a} \quad \text{for $a, \theta\in \{0,1\}$.}
\]
The best response strategy is to take a threshold at $0.5$: 
\[a^*(p) = \ind{p> 0.5} \quad\text{for $p\in [0,1]$}.\] 
The calibration fixed decision loss (CFDL) on decision task $U$ is defined as the improvement in payoff after the decision maker calibrates the predictor:
\begin{align}
    \cfdl_{U}(\vpred, \vstate) = \frac{1}{T}\sum_{t=1}^T\bigg[ U(a^*(\sigma(\pred_t)), \state_t) - U(a^*(\pred_t), \state_t) \bigg],
\end{align}
where $\sigma$ is the mapping that moves each prediction to the actual empirical frequency. 

We can decompose the calibration loss bucket-wise as follows:
\begin{align}
& \cfdl_U(\vact,\vstate) \notag\\
= {} & \frac{1}{T}\sum_{i=1}^m \sum_{t=1}^T \ind{p_t = q_i}\bigg[ \decisionu(a^*(\hat{q}_i), \state_t) - \decisionu(a^*(q_i), \state_t) \bigg]\notag\\
= {} & \frac{1}{T}\sum_{i=1}^m  n_i \E_{\theta\sim \hat q_i}[U(a^*(\hat{q}_i),\theta) - U(a^*(q_i),\theta)]. \qquad\text{($\theta\in \{0,1\}$ is drawn such that $\Pr[\theta = 1] = \hat q_i$)}\label{eq:swap-bucket-intro}
\end{align}

If $q_i$ and $\hat q_i$ belong to the same half of the interval $[0,1]$, i.e., $q_i,\hat q_i\in [0,0.5]$ or $q_i,\hat q_i\in (0.5,1]$, we have $a^*(\hat q_i) = a^*(q_i)$ and thus
\begin{equation}
\label{eq:bucket-intro-1}
\E_{\theta\sim \hat q_i}[U(a^*(\hat q_i),\theta) - U(a^*(q_i),\theta)] = 0.
\end{equation}

If $q_i$ and $\hat q_i$ belong to different halves of the interval, a simple calculation gives
\begin{align}
\E_{\theta\sim \hat q_i}[U(a^*(\hat q_i),\theta) - U(a^*(q_i),\theta)] & = \E_{\theta\sim \hat q_i}[U(a^*(\hat q_i),\theta)] - \E_{\theta\sim \hat q_i}[U(a^*(q_i),\theta)]\notag\\
& =\max(\hat q_i, 1 - \hat q_i) - \min (\hat q_i, 1 - \hat q_i)\notag\\
&= 2 |\hat q_i - 0.5| \notag\\
&\le 2|q_i - \hat q_i|.\label{eq:bucket-intro-2}
\end{align}
Moreover, if $q_i$ and $\hat q_i$ belong to different halves of the interval, by our assumption of $|\hat q_i - q_i| = T^{-1/4}$, we have $|q_i - 0.5| \le T^{-1/4}$. 
Plugging \eqref{eq:bucket-intro-1} and \eqref{eq:bucket-intro-2} into \eqref{eq:swap-bucket-intro}, we get
\begin{equation}
\label{eq:bucket-contrib-intro}
\cfdl_\decisionu (\vact, \vstate) \le \frac{2}{T}\sum_{i=1}^m n_i |q_i - \hat q_i|\ind{|q_i - 0.5| \le T^{-1/4}}.
\end{equation}
Note that the number of $i$'s satisfying $|q_i - 0.5| \le T^{-1/4}$ is $O(T^{1/4})$. Thus, by our assumption of $n_i = \sqrt T$ and $|q_i - \hat q_i| = T^{-1/4}$, we get $\cfdl_\decisionu (\vact, \vstate) = O(T^{-1/2})$.

 While this bound above on $\cfdl$ is for a specific decision task, we can extend it to \emph{every} payoff-bounded decision task and get $\robustcal(\vpred,\vstate) = O(T^{-1/2})$. This follows from a result of \cite{li2022optimization} showing that any payoff-bounded decision task can be expressed as a convex combination of tasks with \emph{V-shaped} payoffs (see \Cref{sec:MSR-Vshape} for more details). Such decision tasks are all very similar to the one we consider here, and we can similarly obtain an $O(T^{-1/2})$ bound for the calibration loss for each of them. This reduction to V-shaped payoffs also played a crucial role in the $O(T^{-1/2})$ U-calibration guarantee of \cite{kleinberg2023u}.

\paragraph{Intuition Generalized.} In our example above, we make the simplifying assumption that the number of predictions in each bucket is the same. However, the $O(T^{-1/2})$ bound on the $\robustcal$ holds without this assumption while only losing a logarithmic factor. That is, we have the following lemma:

\begin{lemma}[Informal special case of \Cref{lm:attribute}]
\label{lm:attribute-informal}
Let $T,m$ be positive integers satisfying $m = \Theta (\sqrt T)$. Define $Q = \{q_1,\ldots,q_m\}\subseteq [0,1]$ where $q_i = i/m$ for every $i = 1,\ldots,m$.
Given a sequence of predictions $\vpred = (p_1,\ldots,p_T)\in Q^T$ and realized states $\vstate = (\theta_1,\ldots,\theta_T)\in \{0,1\}^T$, define $n_i$ and $\hat q_i$ as in \eqref{eq:ni-intro} and \eqref{eq:hat-q-intro}. 

Assume $|\hat q_i - q_i| \le O(\frac 1{\sqrt{n_i}})$ for every $i = 1,\ldots,m$. Then
\[
\robustcal (\vpred, \vstate) \le O(T^{-1/2}\log T).
\]
\end{lemma}

The intuition behind the lemma can be understood by analyzing the contribution of each bucket to $\robustcal$. In $\ece$, as expressed in \Cref{eq:K1-2}, the bias $n_i|q_i - \hat q_i|$ in every bucket contributes to the average, but in our example above, only a minority of the buckets make  positive contribution to the $\robustcal$, as shown in \eqref{eq:bucket-contrib-intro}. In general, for any specific decision task with V-shaped payoff, we show that the contribution to $\robustcal$ from many buckets $i$ is significantly less than the bias $n_i|q_i - \hat q_i|$. This allows us to prove the upper bound on $\robustcal$ in \Cref{lm:attribute-informal}, which would not hold if $\robustcal$ were replaced by $\ece$.

Given \Cref{lm:attribute-informal}, it remains to show that the guarantee $|\hat q_i - q_i| \le O(1/\sqrt n_i)$ can indeed be (approximately) achieved in the online binary prediction setting.
We use the result from \citet{noarov2023highdimensional} which, as stated,  shows an efficient algorithm that gives us a bound only on the maximum of the expectation $\max_i (\expect{}{|\hat q_i - q_i|} - O(1/\sqrt n_i))$, where the expectation is over the randomness in the algorithm.
We  refine their analysis and give a bound on the expectation of the maximum  $\expect{}{\max_i(|\hat q_i - q_i| - O(1/\sqrt n_i))}$ (\Cref{lem:bound LL algorithm 1}). As we show in \Cref{lm:attribute} (a generalized version of \Cref{lm:attribute-informal}), this bound is sufficient for us to obtain $\E[\robustcal(\vpred, \vstate)] \le O(\frac{\log T}{\sqrt{T}})$.
We also note that a simpler minimax proof, similar to the one by \citet{cal-minimax}, also allows us to show the existence of a randomized algorithm that approximately guarantees $|\hat q_i - q_i| \le O(1/\sqrt n_i)$ and thus, by our \Cref{lm:attribute} again, achieves $\E[\robustcal(\vpred, \vstate)] \le O(\frac{\log T}{\sqrt{T}})$ (see \Cref{appdx:minimax}). However, this proof does not come with an explicit construction or any computational efficiency guarantee.  


\subsection{Paper Organization}

We introduce the preliminaries in \Cref{sec:prelim}, including popular measures of the calibration error, decision making and swap regret, and the online binary calibration problem.
%
%
We introduce Calibration Decision Loss ($\robustcal$) in \Cref{sec:MSR}. We discuss its alternative formulation using Bregman divergences and its approximation via V-shaped Bregman divergences, which will be useful to establish our main result.
%
\Cref{sec:connection} discusses the connection between $\robustcal$ and other calibration error metrics. 
%
In \Cref{sec:main-sqrt T}, we present our main result, an efficient online binary prediction algorithm that guarantees $O(\frac{\log T}{\sqrt{T}})$ $\robustcal$. The key technical idea behind this result is a lemma (\Cref{lm:attribute}) we prove in \Cref{sec:attribute} which allows us to attribute $\robustcal$ to bucket-wise biases.
%
%
Additionally, we give a non-constructive but simpler minimax proof of the $O(\frac{\log T}{\sqrt{T}})$ $\robustcal$ guarantee in \Cref{appdx:minimax}. This simpler proof also crucially relies on our key technical lemma (\Cref{lm:attribute}) in \Cref{sec:attribute}.

\subsection{Related Work}
\label{sec:related}

\subsubsection{Calibration Error Metrics}
While perfect calibration has an intuitive and clear definition, it is a non-trivial and subtle question to meaningfully quantify the calibration error of predictions that are not perfectly calibrated. $\ece$ is one of the most popular calibration measures, but it lacks continuity: slightly perturbing perfectly calibrated predictions can significantly increase $\ece$.
To address this issue,
recently \cite{utc} developed a theory of \emph{consistent} calibration measures by introducing the \emph{distance to calibration} as a central notion. This theory has facilitated rigorous explanations of an interesting empirical phenomenon called ``calibration out of the box'' in deep learning
\citep{cal-gap,loss-MC}.

As a relaxation of calibration, \citet{kleinberg2023u}  consider the utility of predictions to downstream decision makers and introduce U-calibration error, the maximum external regret over payoff-bounded decision tasks.  \citet{kleinberg2023u} design an algorithm that achieves an optimal $O(\frac{1}{\sqrt T})$ U-calibration, which is a necessary but insufficient condition for calibration. 

\subsubsection{Online Calibration Algorithms}

Recent research has made significant progress in proving upper and lower bounds on the optimal rate achievable for both $\ece$ and the distance to calibration in online binary prediction, though significant gaps remain between the current best upper and lower bounds. 
For $\ece$ minimization,  \cite{foster1998asymptotic} shows there exists a randomized algorithm that achieves $O(T^{-1/3})$ expected $\ece$, which remains the best known upper bound. Existence proofs and constructions of such algorithms have been further explored in several subsequent works  \citep{cal-minimax,foster2021hedging}. 
A recent work  \citep{sidestep} show a lower bound of $\Omega(T^{-0.472})$ to online $\ece$ minimization.
For the distance to calibration, \cite{qiao-distance} give a non-constructive minimax-based proof for an  $O(\frac{1}{\sqrt T})$ upper bound and an $\Omega(T^{-1/3})$ lower bound for the same problem. Soon afterwards, \cite{elementary} provide a construction of an efficient algorithm that achieves the $O(\frac{1}{\sqrt T})$ upper bound for the distance to calibration.

The literature on online regret minimization for all downstream decision makers is technically closest to our paper.  \citet{kleinberg2023u} observe that the swap regret of actions for any payoff-bounded decision task is linearly upperbounded by $\ece$. This observation quantitatively justifies the qualitative equivalence between no swap regret and calibration in \citet{foster1997calibrated}. However, the $\Omega(T^{-0.472})$ lowerbound presents a barrier in efficient swap regret minimization via $\ece$ minimization. 
To overcome this barrier from $\ece$, several relaxations of $\robustcal$ have been considered by recent works to achieve $\asympO(\frac{1}{\sqrt T})$ regret rates in online binary prediction. 
\citet{kleinberg2023u} achieve the $O(\frac{1}{\sqrt{T}})$ \emph{external} regret maximized over payoff-bounded tasks, a necessary but insufficient condition for asymptotic calibration.
On the other hand, \citet{roth2024forecasting} show $O(|A|\sqrt{\frac{\log T}{T}})$ swap regret bounds that depend additionally on the number of actions $|A|$ in the downstream decision task, which can be loose when a decision task has many or even infinitely many actions. 

\subsubsection{Omniprediction} 
Treating prediction and decision making as separate steps allows us to train a single predictor and use it to solve multiple decision tasks with different utility/loss functions. This separation of training and decision making is the idea behind omniprediction, introduced recently by \cite{omni}, where the goal is to train a single predictor that allows each downstream decision maker to incur comparable or smaller loss than any alternative decision rule from a benchmark class.
Notions from the algorithmic fairness literature (e.g.\ multicalibration and multiaccuracy \citep{mc,ma}) have been used to obtain omnipredictors in various online and offline (batch) settings \citep{omni,loss-oi,constrained,characterize-omni,omni-regression,oracle-omni,noarov2023highdimensional,performative}.
Omniprediction allows better efficiency than training a different model from scratch for each decision task, and it also allows the predictor to be robust to changes in the loss function.



\subsubsection{Swap Regret Minimization} Swap regret minimization algorithms have been studied extensively in the online learning literature \citep[e.g.][]{hart2000simple, hart2001reinforcement, blum-mansour,  hart2013simple, anagnostides2022near}. In game theory, the swap regret is known for its connection to correlated equilibrium. 
\cite{foster1997calibrated} first show vanishing swap regret implies convergence to correlated equilibria.  
Recently, \citet{peng2023fast,  dagan2023external} prove a lowerbound on the swap  regret, which is polynomial in the number of actions. 
Meanwhile, the calibration literature \citep{kleinberg2023u, noarov2023highdimensional, roth2024forecasting} differs from the swap regret minimization literature in two aspects: 1) it focuses on developing a robust strategy that minimizes swap regret simultaneously for all decision makers, and 2) it focuses on minimizing swap regret for the special payoff structure of decision tasks. As a special payoff structure, a decision task restricts the adversary to only be able to select a state. Equivalently, the adversary can select payoff from a low rank matrix with the same rank of the state space. 
Our result also implies the lowerbound on swap regret is strictly weaker when there exists a special low-rank constraint on payoff matrix. 
While the lowerbound on general swap regret minimization depends polynomially on the number of actions, swap regret minimization of decision tasks does not have such dependence given restricted state space.

\subsubsection{Optimization of Scoring Rules} $\robustcal$ is defined as the maximum swap regret over all decision tasks. Since the payoff in a decision task can be equivalently represented by proper scoring rules (see \Cref{sec:prelim-proper}), the computation of $\robustcal$ is an optimization problem of scoring rules. A recent literature \citep{li2022optimization, neyman2021binary, hartline2023optimal} studies the optimization of scoring rules, where \citet{li2022optimization} is the most relevant paper. \citet{li2022optimization} present two results that are helpful to our problem: 1) under their different optimization objective, the optimal scoring rule can be computed via linear programming, and 2) any bounded scoring rule can be decomposed into a linear combination of V-shaped scoring rules. Following their idea, we design a linear program that computes $\robustcal$ in polynomial time. Our $O(\frac{\log T}{\sqrt{T}})$ $\robustcal$ result also uses this linear decomposition of scoring rules (see \Cref{sec:MSR-Vshape}). 


\section{Preliminaries}
\label{sec:prelim}

Throughout the paper, we denote a prediction by $p\in [0,1]$, and a binary state by $\state\in \{0, 1\}$. A prediction $p$ can be viewed as a distribution over the state space $\statesp =  \{0,1\}$, and we write $\theta\sim p$ when we sample $\theta$ from the Bernoulli distribution with mean $p$, i.e., $\Pr[\theta = 1] = p$. For a real number $x$, we use $(x)_+$ or $[x]_+$ to denote $\max\{x,0\}$. We use $\ind{\cdot}$ to denote the 0-1 indicator function: $\ind{\text{statement}} = 1$ if the statement is true, and $\ind{\text{statement}} = 0$ if the statement is false.


\subsection{Measures of Calibration Error}

In this section, we define empirical calibration on $T$ samples and calibration error metrics. Over $T$ samples, the predictions are restricted to fall in a finite space $Q = \{\q_i\in [0, 1]\}_i$. Let $\qcount_i=\sum_t\ind{\pred_t = \q_i}$ be the count of prediction being $\q_i$ in $T$ samples, and  $\empq_i=\frac{\sum_t\ind{\pred_t = \q_i}\state_t}{n_i}$ be the empirical distribution of the realized state conditioning on prediction is $\q_i$.

\begin{definition}
    Given $T$ samples,  predictor is \textit{empirically calibrated} if for each prediction $\q_i\in Q$, the prediction is consistent with its conditional empirical distribution, i.e.\ $\q_i = \empq_i$. 
\end{definition}

\begin{definition}
\label{def:empirical calibrate}
    Given $T$ samples, we can \textit{empirically calibrate} a prediction by swapping predictions to their conditional empirical frequencies, i.e.\ by applying swap mapping $\sigma^*$ to predictions with $\sigma^*(\q_i) = \empq_i$. 
\end{definition}

While there is only one natural and clear definition of perfect calibration, there are various metrics for measuring the calibration error. 
\cite{foster1998asymptotic} defines the  expected calibration error  (ECE) which is a measure of calibration error popularly used in the literature. $\ece$ measures the average absolute distance between the prediction and the empirical distribution. 

\begin{definition}[$\ece$]
    Given $T$ samples of predictions $\vpred = (\pred_t)_{t\in [T]}$ and corresponding realizations $\vstate  = (\state_t)_{t\in [T]}$ of states. $\ece$ is
   \begin{equation*}
       \ece(\vpred, \vstate) = \frac{1}{T}\sum_{t\in [T]}|\pred_t - \sigma^*(\pred_t)|.
   \end{equation*}
Equivalently, $\ece = \frac{1}{T}\sum_{\q_i\in Q} \qcount_i |\q_i - \empq_i|$.
\end{definition}

 An alternative metric is the $\lcal{2}$ calibration error, the average squared distance between a prediction and the empirical distribution. \Cref{thm:k2 sqrt T} shows it is possible to achieve $\asympO(T^{-\sfrac{1}{2}})$ worst-case expected $\lcal{2}$ error. 

\begin{definition}[$\lcal{2}$ calibration error]
    Given  $T$ samples of predictions $\vpred = (\pred_t)_{t\in [T]}$ and corresponding realizations $\vstate  = (\state_t)_{t\in [T]}$ of states, the $\lcal{2}$ calibration error is
   \begin{equation*}
       \lcal{2}(\vpred, \vstate) = \frac{1}{T}\sum_{\q_i\in Q} \qcount_i (\q_i - \empq_i)^2.
   \end{equation*}
\end{definition}

In addition to $\ece$ and $\lcal{2}$ calibration error, we also compare to the smooth calibration error introduced by \cite{smooth}. 
Unlike $\ece$ and $\lcal{2}$, the smooth calibration error is continuous in predictions.
\begin{definition}[Smooth Calibration Error, \citealp{smooth}]\label{def:smooth calibration error}
    Given  $T$ samples of predictions $\vpred = (\pred_t)_{t\in [T]}$ and corresponding realizations $\vstate  = (\state_t)_{t\in [T]}$ of states. The smooth calibration error is a supremum over the set $\Sigma$ of $1$-Lipschitz functions $\sigma:[0, 1]\to [-1, 1]$:
    \begin{equation*}
        \smooth(\vpred, \vstate) = \frac{1}{T}\sup_{\sigma\in \Sigma}\sum_t \sigma(\pred_t)(\pred_t - \state_t).
    \end{equation*}
\end{definition}
The definition above is equivalent to $\lcal{1}$ without the $1$-Lipschitz constraint on $\sigma$. Taking the following non-Lipschitz $\sigma$ yields $\ece$.
\begin{equation*}
    \sigma(\q_i) = \begin{cases}
     1,    &  \text{if }\q_i - \empq_i\geq 0;\\
    -1,     & \text{otherwise.}
    \end{cases}
\end{equation*}

The smooth calibration error is polynomially related to the distance to calibration, which measures the absolute distance to the closest calibrated prediction. 

\begin{definition}[Distance to Calibration, \citealp{utc}]
      Given $T$ samples of  predictions $\vpred = (\pred_t)_{t\in [T]}$ and corresponding realizations $\vstate  = (\state_t)_{t\in [T]}$ of states. The distance to calibration is
      \begin{equation*}
          \distcal(\vpred, \vstate) = \frac{1}{T}\min_{\vempp:\lcal{1}(\vempp, \vstate) = 0} \sum_t|\pred_t - \empp_t|.
      \end{equation*}
\end{definition}

\begin{lemma}[\citealp{utc}]
    Smooth calibration error $\smooth$ is polynomially related to distance to calibration $\distcal$.\:
    \begin{equation*}
       \smooth\leq \distcal\leq \sqrt{32\, \smooth}.
    \end{equation*}
\end{lemma}

\subsection{Decision Making and Swap Regret
}
\label{sec:prelim-swap regret}



Before introducing predictions and calibration,  we define  decision making and swap regret. 
%
A decision task is defined by three components. Throughout the paper, we normalize the payoff in the decision task to $[0, 1]$. 
 \begin{itemize}
     \item The agent takes action $\act\in \actsp$.
     \item A payoff-relevant state $\state\in \statesp$ realizes.
     \item The agent obtains payoff $\decisionu:\actsp\times\statesp\to [0, 1]$ as a function of the action and the state.\footnote{$\decisionu$ can be an arbitrary function. }
 \end{itemize}

 The decision is evaluated by the average performance when the decision maker repeatedly faces an identical decision task $(\actsp, \statesp, \decisionu)$ in $T$ rounds. We define the swap regret the same as in \citet{roth2024forecasting}. 

 \begin{definition}[Swap Regret]
 \label{def:action-swap}
    Given a sequence of $T$ actions $\vact = (\act_t)_{t\in [T]}$ and realization $\vstate  = (\state_t)_{t\in [T]}$ of states, and fix a decision task with payoff rule $\decisionu$, the swap regret of the decision maker is 
    \begin{equation*}
        \actswap_\decisionu(\vact, \vstate) = \frac{1}{T}\max_{\sigma:\actsp\to\actsp}\sum_{t\in[T]}\bigg[ \decisionu(\sigma(\act_t), \state_t) - \decisionu(\act_t, \state_t) \bigg].
     \end{equation*}
\end{definition}

\cite{kleinberg2023u} aim to minimize the external regret for all decision tasks with bounded payoff. We define the external regret here for comparison. 

\begin{definition}[External Regret]
     Given a sequence of $T$ actions $\vact = (\act_t)_{t\in [T]}$ and realization $\vstate  = (\state_t)_{t\in [T]}$ of states, and fix a decision task with payoff rule $\decisionu$, the external regret is calculated against the best fixed action,
     \begin{equation}
         \external_\decisionu(\vact, \vstate) = \frac{1}{T}\max_{\act\in\actsp} \bigg[ \decisionu(\act, \state_t) - \decisionu(\act_t, \state_t) \bigg].
     \end{equation}
\end{definition}

\cite{kleinberg2023u} define the U-calibration error, the maximum external regret when the decision maker best responds to the predictions.  Note that no U-calibration error is necessary but insufficient for calibration. 

\begin{definition}[U-calibration Error]
  Let $\vpred = (\prob_t)_{t\in [T]}$ be a sequence of $T$ predictions  and let $\vstate  = (\state_t)_{t\in [T]}$ be the realization  of states. Consider a decision maker with payoff function $\decisionu:A\times \{0,1\}\to [0,1]$ who best responds to the predictions by taking $\act_t = \bestact(\pred_t)=\argmax_{\act\in\actsp}\expect{\state\sim \prob_t}{\decisionu(\act, \state)}$.  Let $\vact_U = (a_1,\ldots,a_T)$ be the vector of best-response actions. The U-calibration error is the maximum external regret over decision tasks with bounded payoff in $[0, 1]$:
    \begin{equation*}
        \ucal(\vpred, \vstate) = \sup_{\decisionu}\external_\decisionu(\vact_\decisionu, \vstate),
    \end{equation*}
   where the supremum is over all payoff functions $U:A\times \{0,1\}\to [0,1]$ with arbitrary action spaces $A$.
\end{definition}

\subsection{Payoff as Proper Scoring Rules}

\label{sec:prelim-proper}
For ease of notation throughout the paper, we will write the payoff as a function of the prediction assuming the agent trusts the prediction and best responds to it. In this section, we introduce this payoff as a function of the prediction and show its equivalence to proper scoring rules. 

When the agent is assisted with a sequence of predictions, the agent obtains payoff by acting in response to the prediction.  
If the prediction is calibrated, the best response $\bestact$ is a function of the prediction:
\begin{equation}\label{eq:best response}
\bestact(\pred)=\argmax_{\act\in\actsp}\expect{\state\sim \prob}{\decisionu(\act, \state)}.
\end{equation}

The payoff is equivalently a function of the prediction. 
We present the payoff in a scoring rule $\score:\Delta(\statesp)\times \statesp\to \R$, representing the utility directly as a function of the prediction. 
\begin{definition}[Scoring rule induced by decision task] Given a decision task $(\actsp, \statesp, \decisionu)$ and its corresponding best response function $a^*$, we define an induced scoring rule $S_U:\Delta(\statesp)\times \statesp\to \R$ such that for any prediction $\pred\in \Delta(\statesp)$ and state $\state\in \statesp$,
\[
S_U(\pred,\state) = U(a^*(\pred), \state).
\]
That is, $S_U(p,\state)$ is the payoff of a decision maker who chooses the best response $a^*(p)$ based on the prediction $p$.
\end{definition}

\citet{frongillo2014general} observe that scoring rules induced by a decision task are equivalent to the class of proper scoring rules \citep{Mcc-56}. 
Proper scoring rules are defined such that the expected score is maximized when the prediction is the true distribution. Evaluated with a proper scoring rule, the predictor has an incentive to truthfully report the distribution in order to maximize expected score. 

\begin{definition}[Proper Scoring Rule]
    A scoring rule $\score$ is proper if for any $\pred'\in\Delta(\statesp)$ that is not the true distribution $\state$ is generated,
    \begin{equation*}
        \expect{\state\sim\pred}{\score(\pred, \state)}\geq \expect{\state\sim\pred}{\score(\pred', \state)}.
    \end{equation*}
\end{definition}

\Cref{claim:decision-proper} shows that 1) the scoring rule induced by a decision task is a proper scoring rule, and 2) any proper scoring rule measures the best-response payoff in a decision task. 

\begin{claim}\label{claim:decision-proper}
Proper scoring rules and payoffs in decision task are equivalent:
\begin{enumerate}
    \item Any scoring rule $\score_\decisionu$ induced by a decision task $(\actsp, \statesp, \decisionu)$ is proper. 
    \item For any proper scoring rule $\score$, there exists a decision task with payoff $\decisionu$, such that $\decisionu(\bestact(\pred), \state) = \score(\pred, \state)$ for every prediction $\pred$.
\end{enumerate}

\end{claim}

\begin{proof}
    2 is straight forward by defining $\actsp = \Delta(\statesp)$,  the action space as the prediction space. 

    To see 1, notice that the score is the best-response score with $S_U(\pred,\state) = U(a^*(\pred), \state)$. Hence, if the prediction is  $\pred'\neq \pred$ not the distribution where $\state$ is drawn, the agent obtains a payoff weakly lower than the payoff by best responding to the true distribution. 
\begin{equation*}
    \expect{\state\sim\pred}{\score(\pred', \state)}=\expect{\state\sim\pred}{\decisionu(\bestact(\pred'), \state)}\leq \expect{\state\sim\pred}{\decisionu(\bestact(\pred), \state)}=\expect{\state\sim\pred}{\score(\pred, \state)}.\qedhere
\end{equation*}
\end{proof}

Following this equivalence between proper scoring rules and best-respond payoffs, we will introduce our results in a proper scoring rule $\score$ instead of the decision payoff $\decisionu$.

\subsection{Calibration Fixed Decision Loss}

In the previous section, we assume that the agent trusts the prediction and best responds to it. If, however, the agent best responds to  miscalibrated predictions, she suffers the Calibration Fixed Decision Loss ($\cfdl$), the improvement in payoff when she empirically calibrates the predictions.

\begin{definition}[$\cfdl$]
     Given  $T$ samples of predictions $\vpred = (\pred_t)_{t\in [T]}$ and corresponding realizations $\vstate  = (\state_t)_{t\in [T]}$ of states, fixing a proper scoring rule $\score$, the Calibration Fixed Decision Loss ($\cfdl$) is
     \begin{equation*}
         \cfdl_\score =  \frac{1}{T}\sum_{t}\bigg[ \score(\sigma^*(\pred_t), \state_t) - \score(\pred_t, \state_t) \bigg],
     \end{equation*}
     where $\sigma^*$ is the swap mapping that empirically calibrates the predictor (see \Cref{def:empirical calibrate}).
\end{definition}

The $\cfdl$ has an equivalent formalization as prediction swap regret, where the agent is allowed to swap predictions in hindsight. Note that the prediction swap regret is weakly stronger than the swap regret (\Cref{def:action-swap}) on the same decision task. 

\begin{definition}[Prediction Swap Regret]
\label{def:prediction swap}
    Given a sequence of $T$ predictions $\vpred = (\prob_t)_{t\in [T]}$ and realization $\vstate  = (\state_t)_{t\in [T]}$ of states, fixing a proper scoring rule $\score$, the prediction swap regret is 
    \begin{equation*}
        \predswap_S(\vpred, \vstate) = \max_{\sigma:\Delta(\statesp)\to\Delta(\statesp)}\frac{1}{T}\sum_{t\in[T]}\bigg[ \score(\sigma(\pred_t), \state_t) - \score(\pred_t, \state_t) \bigg].
     \end{equation*}
\end{definition}

\Cref{prop:cfdl eq pswap} shows $\cfdl$ is equivalent to prediction swap regret. 

\begin{proposition}
\label{prop:cfdl eq pswap}
    Given a sequence of $T$ predictions $\vpred = (\prob_t)_{t\in [T]}$ and realization $\vstate  = (\state_t)_{t\in [T]}$ of states, fixing a proper scoring rule $\score$, CFDL equals the prediction swap regret:
    \begin{equation*}
        \cfdl_\score = \predswap_\score.
    \end{equation*}
\end{proposition}

\Cref{prop:cfdl eq pswap} follows directly from the properness of scoring rules. Conditioning on the prediction is $\q_i$, the state follows the empirical distribution:
\begin{equation*}
    \frac{1}{\qcount_i}\sum_{t\in [T]}\score(\pred, \state_t)\ind{\pred_t = \q_i} = \expect{\state\sim\empq_i}{\score(\pred, \state)}.
\end{equation*}By properness of the scoring rule, predicting the empirical distribution maximizes the expected  score (average score). Thus, the optimal swap function in \Cref{def:prediction swap} is $\sigma^*$ which empirically calibrates the predictor.


The prediction swap regret is stronger than the swap regret in \Cref{def:action-swap}. If an algorithm generates a sequence of predictions with low prediction swap regret, then the agent has low  swap regret if they best respond to the predictions.  To see this, notice that the modification rule in prediction swap regret has more power.  The prediction swap regret allows the agent to modify the action conditioning on each prediction. If two predictions have the same best-response action, the swap regret does not allow the agent to apply different modification to the two predictions. 

\begin{claim}
\label{claim:action-prediction}
Given a decision task with payoff rule $\decisionu$, the corresponding scoring rule is denoted $\score$. If the agent best responds by taking $a_t = a^*(p_t)$ at each round $t$, 
\begin{equation*}
\actswap_U(\vact,\vstate)\le \predswap_{\score_\decisionu}(\vpred,\vstate)
\end{equation*}

\end{claim}

\begin{claim}[{\citet{kleinberg2023u}, Theorem 12}]\label{claim:kleinberg k1 swap bound}
\footnote{\cite{kleinberg2023u} assume that the output of the scoring rule $S$ is in $[-1,1]$, whereas we assume the output is in $[0,1]$. Thus, the constant $4$ in their bound translates to the constant $2$ here.}For any proper scoring rule $S:[0,1]\times \{0,1\}\to [0,1]$ and any sequences $\vpred = (\pred_1,\ldots,\pred_T)\in [0,1]^T, \vstate = (\state_1,\ldots,\state_T)\in \{0,1\}^T$, it holds that
\[
\cfdl_{S}(\vpred,\vstate) = \predswap_{S}(\vpred,\vstate) \le 2\ece(\vpred, \vstate).
\]
\end{claim}

 $\lcal{2}$ calibration error can be written as a special case of $\cfdl$ for quadratic scoring rule (a.k.a. squared loss). In fact, the prediction swap regret was first introduced in \citet{foster1998asymptotic} with the quadratic scoring rule $\score(\pred, \state) = 1-(\pred - \state)^2$. 

\begin{lemma}\label{lem:k2-quadratic}
Define the quadratic scoring rule $\score_2(\pred, \state) = 1-(\pred - \state) ^2$. We have
\begin{equation*}
    \lcal{2} = \cfdl_{\score_2}.
\end{equation*}
\end{lemma}

$\ece$ cannot be represented by the $\cfdl$ with any proper scoring rule:
\begin{lemma}\label{lem:k1-untruthful}
    There does not exist a proper scoring rule $\score$, such that for any sequence of predictions $\vpred$ and states $\vstate$,
    \begin{equation*}
        \ece(\vpred, \vstate) = \cfdl_{\score}(\vpred, \vstate).
    \end{equation*}
\end{lemma}

\Cref{lem:k2-quadratic} follows immediately from the definitions of $\lcal{2}$ and $S_2$, whereas \Cref{lem:k1-untruthful} can be proved using the Bregman divergence characterization of proper scoring rules (see \Cref{prop:swap-bregman}).



\subsection{Online Binary Calibration}
\label{sec:online}

We focus on the classic online prediction problem studied by \cite{foster1998asymptotic}. The goal is to generate calibrated predictions in the long run, even if the states are adversarially selected. Unless otherwise specified, we present results under the binary prediction setting, i.e.\ the state space is $\statesp = \{0, 1\}$. 
In this setup, a predictor (algorithm) $F$ makes a prediction $\pred_t\in [0,1]$ at each round $t = 1,2,\ldots$, and an adversary $A$ picks a binary state (outcome) $\state_t\in \{0,1\}$. 
Both the prediction $\pred_t$ and the state $\state_t$ can depend on the past history $h_{t - 1} = (\pred_1, \state_1, \dots, \pred_{t-1}, \state_{t-1})$, but they cannot depend on each other. That is, we can assume without loss of generality that the algorithm chooses $p_t$ and the adversary reveals $\theta_t$ simultaneously. The transcript $h_t = (p_1,\theta_1,\ldots,p_t,\theta_t)$ is a function of the strategies of the predictor $F$ and the adversary $A$. That is, $h_t = h_t(F,A)$. The predictions in $T$ rounds are evaluated by a calibration error metric.

We allow the predictor to be randomized, in which case we can view its strategy as a distribution $\cF$ over deterministic strategies $F$. When we use an error metric, say $\ece$, to evaluate the predictions made by our predictor in $T$ rounds, our goal is to minimize the expected value of the error metric w.r.t.\ the worst-case adversary $A$, i.e., we want $\expect{F\sim \cF}{\robustcal(h_T(F,A))}$ to be small for every adversary $A$.

Previous results for different calibration errors are listed here. 

\begin{description}
    \item[$\ece$] There exists a gap between the known upperbound and lowerbound for $\ece$.
    \begin{description}
        \item[Upperbound] \begin{theorem}[\citealp{foster1998asymptotic}, see also \citealp{cal-minimax,foster2021hedging}]
There exists a randomized online binary prediction algorithm that guarantees $O(T^{-\sfrac{1}{3}})$ expected $\ece$. 
\end{theorem}
        \item[Lowerbound] \begin{theorem}[\citealp{sidestep}]
For any randomized online binary prediction algorithm, the $\ece$ w.r.t.\ the worst-case adversary is $\Omega(T^{-0.472})$.
\end{theorem}
    \end{description}


\item[$\lcal{2}$ Calibration Error] 
As a variant of $\ece$, $\lcal{2}$ calibration error has an upperbound result from the literature. 

\begin{description}
    \item[Upperbound] 
    \begin{theorem}[\citealp{rothbook}]
\label{thm:k2 sqrt T}
There exists a randomized online binary prediction algorithm that guarantees $O(T^{-\sfrac{1}{2}}\log T)$ expected $\lcal{2}$ calibration error. 
\end{theorem}


\end{description}

\item[Smooth Calibration Error] \cite{utc} propose the smooth calibration error. 

\begin{description}
    \item[Upperbound] \cite{qiao-distance} prove the existence of an algorithm that achieves $O(\frac{1}{\sqrt{T}})$  distance to calibration, following which  \cite{elementary} give a construction of such an algorithm. The result implies $O(\frac{1}{\sqrt{T}})$ smooth calibration error, since the distance to calibration upperbounds the smooth calibration error. 

\begin{theorem}[\citealp{qiao-distance, elementary}]
    There exists a randomized online binary prediction algorithm that guarantees $O(\frac{1}{\sqrt{T}})$ expected distance to calibration. 
\end{theorem}
\item [Lowerbound]
\begin{theorem}[\citealp{qiao-distance}]
    For any randomized online binary prediction algorithm, the smooth calibration error w.r.t.\ the worst-case adversary is $\Omega(T^{-\frac{2}{3}})$.
\end{theorem}
\end{description}

\end{description}


\section{Calibration Decision Loss}

\label{sec:MSR}

In this section, we introduce our calibration error metric, the Calibration Decision Loss (CDL), meaning the decision loss resulted from miscalibration. In \Cref{sec:MSR-breg,sec:MSR-Vshape}, we discuss connections between $\robustcal$ and Bregman divergences that will be useful for obtaining our main result in \Cref{sec:main-sqrt T}. We show an efficient algorithm for computing $\robustcal$ in \Cref{sec:computation}.


\begin{definition}[$\robustcal$]
    Given a sequence of $T$ predictions $\vpred = (\prob_t)_{t\in [T]}$ and realization $\vstate  = (\state_t)_{t\in [T]}$ of states, we define the Calibration Decision Loss ($\robustcal$) as
    \begin{equation}\label{eq:cdl}
        \robustcal(\vpred, \vstate) = \sup_{\score\in [0, 1]} \cfdl_S(\vpred, \vstate),
    \end{equation}
    where the supremum is over all proper scoring rules $\score$ with range bounded in $[0, 1]$.
\end{definition}

The $\robustcal$ as defined above is equal to the maximum swap regret (\Cref{def:action-swap}) of best responding decision makers with payoffs bounded in $[0,1]$:
\begin{lemma}
 Let $\vpred = (\prob_t)_{t\in [T]}$ be a sequence of $T$ predictions  and let $\vstate  = (\state_t)_{t\in [T]}$ be the realization  of states. Consider a decision maker with payoff function $\decisionu:A\times \{0,1\}\to [0,1]$ who best responds to the predictions by taking $\act_t = \bestact(\pred_t)=\argmax_{\act\in\actsp}\expect{\state\sim \prob_t}{\decisionu(\act, \state)}$. Let $\vact_U = (a_1,\ldots,a_T)$ be the vector of best-response actions. We have
    \begin{equation*}
        \robustcal(\vpred, \vstate) = \sup_{\decisionu}\actswap_\decisionu(\vact_U, \vstate),
    \end{equation*}
where the supremum is over all payoff functions $U:A\times \{0,1\}\to [0,1]$ with arbitrary action spaces $A$.
\end{lemma}

\begin{proof}
    For any decision task with payoff function $U$, the prediction swap regret for the corresponding scoring rule $S$ is higher than the swap regret: $\actswap_\decisionu(\vact_U, \vstate) \le \swapS_{\score_\decisionu}(\vpred, \vstate)$ (\Cref{claim:action-prediction}). Thus, $\robustcal\geq \sup_{\decisionu}\actswap_\decisionu(\vact_U, \vstate)$. On the other hand, for any scoring rule $\score$,  construct a decision task with payoff $\decisionu$ as in the proof of \Cref{claim:decision-proper}, by setting $\actsp = [0, 1]$. The resulting decision task has $\swapS_{\score} = \actswap_{\decisionu}$. 
Thus we have  $\robustcal= \sup_{\decisionu}\actswap_\decisionu(\vact_U, \vstate)$.
\end{proof}

By definition, if $\robustcal$ vanishes, then $\cfdl$ (and thus swap regret) also vanishes for every downstream agent. 

\begin{proposition}
    If $\robustcal = \asympO(\frac{1}{\sqrt{T}})$, then both $\cfdl = \asympO(\frac{1}{\sqrt{T}})$ and $\actswap = \asympO(\frac{1}{\sqrt{T}})$ simultaneously for every downstream decision task with  payoff bounded in $[0, 1]$. 
\end{proposition}

\subsection{Bregman Divergence Formulation}
\label{sec:MSR-breg}
As an important preparation for our main result in \Cref{sec:main-sqrt T}, we show that the $\cfdl$ can be written in an equivalent form using the Bregman divergence (\Cref{prop:swap-bregman}). This follows from the convex function formulation of proper scoring rules from \citet{Mcc-56, sav-71}. 


Each proper scoring rule induces a Bregman divergence, measuring the loss from incorrect prediction. The Bregman divergence is defined with the convex function of a scoring rule. See \Cref{fig:proper-score} for a  geometric demonstration of proper scoring rules.

\begin{theorem}[\citealp{Mcc-56, sav-71}]
\label{thm:score-u}
    A scoring rule is proper if and only if there exists a convex function $\util:[0,1]\to\reals$ and its sub-gradient $\nabla\util:[0,1]\to \reals$ such that
\begin{equation}
\label{eq:score-u}
    \score(\pred, \state) = \util(\pred) + \nabla\util(\pred)\cdot(\state- \pred).
\end{equation}
\end{theorem}

Notice that when the prediction is the true distribution, $\expect{\state\sim\pred}{\score(\pred, \state)} = \util(\pred)$ since $\expect{}{\state} = \pred$.

\begin{definition}[Bregman divergence]
    Specified by a convex function $\util:[0,1]\to \reals$ and its sub-gradient $\nabla \util:[0,1]\to \reals$, the Bregman divergence $\breg(\pred, \hat \pred)$ is defined as\footnote{The definition of Bregman divergence here reverses the order of input in the conventional definition to align with proper scoring rules. In the conventional definition, $\breg(\pred, \hat{\pred}) = \util(\pred) - \util(\hat{\pred}) + \nabla\util(\hat{\pred})(\hat{\pred} - \pred)$.}
\begin{equation}
\label{eq:breg}
    \breg(\pred, \hat{\pred}) = \util(\hat{\pred}) - \util(\pred) + \nabla\util(\pred)(\pred - \hat{\pred}).
\end{equation}
\end{definition}

Given a proper scoring rule $\score$, we write the Bregman divergence defined with $\util(\pred) = \expect{\state\sim\pred}{\score(\pred, \state)}$ as $\breg_\score$. Breaking down $\cfdl$ to each prediction value in $Q$, it follows that the contribution of $\q_i\in Q$ is:
\begin{equation*}
    \frac{1}{\qcount_i}\sum_{t\in [T]}(\score(\empq_i, \state_t) - \score(\q_i, \state_t) )\ind{\pred_t = \q_i} = \expect{\state\sim\empq_i}{\score(\empq_i, \state) - \score(\q_i, \state)}=\breg_\score(\q_i, \empq_i).
\end{equation*}

\begin{figure}[thbp]
    \centering
          \begin{tikzpicture}[scale = 0.50]
      
      \draw[scale=1, domain=0:1.0, smooth, variable=\t, ultra thick]
  plot ({\t * 10},{(\t^2 + 1 - \t) * 5});

        \draw [white] (0, 0) -- (11.5, 0);
        \draw (0,0) -- (10.5, 0);
        \draw (0, 0) -- (0, 5.5);

\draw[scale = 1, domain = 0:1, variable = \x] plot ({\x * 10}, {(0.75^2 + 1 - 0.75 + (2 * 0.75 - 1) * (\x - 0.75)) * 5});
        \draw[dotted] (0, 5) -- (10, 5);
        \draw[dotted] (7.5, 0) -- (7.5, 4.0625);
        \draw (7.5, -0.5) node {$\pred$};

        \node [circle, fill = black, inner sep=0pt, minimum size=3.5pt, label = left: {$\score(\pred, 0)$}] at (0, 2.1875){};
        \node [circle, fill = black, inner sep=0pt, minimum size=3.5pt, label = right: {$\score(\pred, 1)$}] at (10, 4.6875){};
        
        \node [circle, fill = black, inner sep=0pt, minimum size=3.5pt, label = north: {$\util(\pred)$}] at (7.5, 4.0625){};

        \draw (2, 0) -- (2, 0.2);
        \draw [dotted] (2, 0) -- (2, 4.2);
        \draw (2, -0.5) node {$\empp$};

        \draw [dotted] (2, 4.2) -- (0, 4.2);
        \draw[dotted] (2, 2.6875) -- (0,  2.6875);        \node [circle, fill = black, inner sep=0pt, minimum size=3.5pt, label = north: {$\util(\empp)$}] at (2, 4.2){};     

        \draw [decorate,decoration={brace,amplitude=3pt,mirror, raise = 1.5pt}]
  (0, 4.2) -- (0,  2.6875) node[midway, xshift = -35pt]{$\breg(\pred, \empp)$};

        \draw (-0.5, 6) node {$\score$};

        \draw (-0.38, 0) node {\small $0$};
        \draw (-0.38, 5) node {\small $1$};
        \draw (10, -0.5) node {\small $1$};
        \draw (0, -0.5) node {$0$};
        \draw (10, 0) -- (10, 0.2);

        \draw (5,-1.3) node {state; report};
        \draw  (10.5,2.5) node [rotate=90] {score};
        
      \end{tikzpicture}     
    \caption{The graphic explanation of the connection between proper scoring rule and Bregman divergence. The thick convex curve plots the convex utility function $\util(\pred)$ for a proper scoring rule. Fix a report, the score $\score(\pred, \state) = \util(\pred) + \nabla\util(\pred) (\state - \pred)$ is the extreme points on the gradient hyperplane passing $\util(\pred)$ (the thin line). Given empirical distribution $\empp$, the Bregman divergence $\breg(\pred, \empp)$ is the loss of reporting $\pred$ instead of $\empp$.}
    \label{fig:proper-score}
\end{figure}

\begin{proposition}\label{prop:swap-bregman}
 Given a sequence of $T$ predictions $\vpred = (\pred_t)_{t\in [T]}$ from a finite set $Q = \{q_1,\ldots,q_m\}\subseteq [0,1]$ and realizations $\vstate  = (\state_t)_{t\in [T]}$ of states, define $n_i$ and $\hat q_i$ as in \eqref{eq:ni-intro} and \eqref{eq:hat-q-intro}. For any  proper scoring rule $\score$,
    \begin{equation}
    \label{eq:swapS-1}
        \cfdl_\score(\vpred, \vstate) = \frac{1}{T}\sum_{i\in [m]}n_i\breg_\score(q_i, \empq_i).
    \end{equation}
\end{proposition}

\subsection{Approximation via V-Bregman Divergences}
\label{sec:MSR-Vshape}
Following ideas from \citet{li2022optimization, kleinberg2023u}, we decompose any Bregman divergence into a linear combination on a basis. Thus, our $\robustcal$ is bounded by the worst case $\cfdl$ on the basis, the V-Bregman divergences. 

\begin{definition}[V-Bregman divergence]
     The V-Bregman divergence with kink $\kink$ is defined as 
   \begin{equation}
   \label{eq:vbreg}
       \vbreg_{\kink}(q, \hat{q}) = \begin{cases}
        0,    &  \text{if 1) } q< \kink \text{ and }\hat{q}\leq \kink \text{ or 2)}q\geq \kink \text{ and }\hat{q}\geq \kink;\\
         \frac{|\hat{q}-\kink|}{\max\{1-\kink, \kink\}},    & \text{otherwise}.
       \end{cases}
   \end{equation}
Equivalently, this is the Bregman divergence for the convex function $u:[0,1]\to [0,1]$ shown in \Cref{fig: v shape}, whose sub-gradient is 
\[
\nabla u (p) = \begin{cases}
-\frac{1}{2\max\{1 - \mu,\mu\}}, & \text{if }p < \mu;\\
\frac{1}{2\max\{1 - \mu,\mu\}}, & \text{if }p \ge \mu.
\end{cases}
\]
\end{definition}

We note that any V-Bregman divergence can be induced by a V-shaped proper scoring rule bounded in $[0, 1]$. More specifically, the V-shaped scoring rule offers two actions for the agent to decide, with optimal decision rule as a threshold at $\kink$. Without loss of generality, suppose $\kink\leq \sfrac{1}{2}$. $\vbreg_\kink$ can be induced from the following proper scoring rule, which is shown in \Cref{fig: v shape}.
 \begin{equation}
 \label{eq:V-shaped score}
      \score_{\kink} (\pred, \state) = \left\{\begin{array}{cc}
      \sfrac{1}{2} -\frac{1}{2}\cdot \frac{\state - \kink}{1-\kink}  &  \text{if }\pred\leq \kink\\
        \sfrac{1}{2} +\frac{1}{2}\cdot \frac{\state - \kink}{1-\kink}    & \text{else},
      \end{array}
      \right.
   \end{equation}

   \begin{figure}[thbp]
    \centering
      \begin{tikzpicture}[scale = 0.50]

        \draw [white] (0, 0) -- (11.5, 0);
        \draw (0,0) -- (10.5, 0);
        \draw (0, 0) -- (0, 5.5);

        \draw [ultra thick] plot (0, 3.57) -- (3, 2.5);
        \draw[ultra thick] plot (3, 2.5)-- (10, 5);
        \draw (0, 3.57) -- (10, 0);
        \draw (0, 1.428) -- (10, 5);

        \draw[dotted] (0, 5) -- (10, 5);
        \draw[dotted] (0, 2.5) -- (10, 2.5);

        \draw[dotted] (3, 0) -- (3, 2.5);




        \draw (-0.5, 6) node {$\score$};
        \draw (-0.38, 0) node {\small $0$};
        \draw (-0.38, 5) node {\small $1$};
        \draw (-0.5, 2.5) node {$\sfrac{1}{2}$};
        \draw (10, -0.5) node {\small $1$};
        \draw (0, -0.5) node {$0$};
        \draw (3, -0.6) node {$\kink$};
        \draw (3, 0) -- (3, 0.2);
        \draw (10, 0) -- (10, 0.2);

        \draw (-1.5,  1.428) node {$\score(1, 0)$};
        \draw (-1.5, 3.57) node {$\score(0, 0)$};
        \draw (11.6, 0) node {$\score(0, 1)$};
        \draw (11.6, 5) node {$\score(1, 1)$};

        \draw (5,-1.3) node {state; report};
        \draw  (10.5,2.5) node [rotate=90] {score};
        
      \end{tikzpicture}     
    \caption{The thick black line plots the special convex  function $\util$ for the scoring rule. The convex utility function is V-shaped, consisting of two linear pieces intersecting at $\kink$. Once fixing the prediction $\pred$, the score $\score(\pred, \state)= \util(\pred) + \nabla\util(\pred)\cdot(\state - \pred)$  is linear in the state $\state$. The scoring rule offers two set of scores for the prediction to selecttwo linear lines for prediction. When $\pred\leq q_0$, the prediction selects scores  $\{\score(0, 0), \score(0, 1)\}$. Otherwise, the prediction selects  $\{\score(1, 0), \score(1, 1)\}$.}
    \label{fig: v shape}
\end{figure}



The following lemma allows us to decompose a general Bregman divergence as a linear combination of V-Bregman divergences.
\begin{lemma}
[\citet{li2022optimization, kleinberg2023u}]
\label{lm:v-decompose}
Let $u:[0,1]\to [0,1]$ be a twice differentiable convex function. Let $S$ be the proper scoring rule generated by $u$ as in \eqref{eq:score-u}. Then for every $p,\hat p\in [0,1]$,
    \begin{equation}
    \label{eq:linear-decomp-breg}
        \breg_\score (p,\hat p)  = \int_{\kink = 0}^{1} \util''(\kink)\cdot\max(1-\kink, \kink)\vbreg_{\kink}(p,\hat p)\diff \kink.
    \end{equation}
\end{lemma}


\begin{definition}[$\swapV$]
         Consider a sequence of $T$ predictions $\vpred = (\pred_t)_{t\in [T]}$ from a finite set $Q = \{q_1,\ldots,q_m\}\subseteq[0,1]$ and realizations $\vstate  = (\state_t)_{t\in [T]}$ of states.
         For $\mu\in [0,1]$, we define $\cfdl_\mu(\vpred,\vstate)$ to be the Calibration Fixed Decision Loss for the V-shaped scoring rule $\score_\kink$ with kink $\mu$, or equivalently, by \Cref{prop:swap-bregman}, $\cfdl_\mu(\vpred,\vstate)$ can be expressed using the V-Bregman divergence at kink $\mu$:
         \[
         \cfdl_\mu(\vpred,\vstate) = \frac{1}{T}\sum_{i\in [m]}n_i\vbreg_{\kink}(q_i, \empq_i).
         \]
         We define $\swapV(\vpred,\vstate)$ to be the supremum of $\cfdl_\mu(\vpred,\vstate)$ over all V-Bregman divergences:
\begin{equation}
\label{eq:swapV}
    \swapV(\vpred,\vstate) = \sup_{\kink\in[0, 1]}\cfdl_\mu(\vpred,\vstate).
\end{equation}
\end{definition}

\begin{theorem}\label{thm:v approx msr}
$\swapV$ is a constant-factor approximation of $\robustcal$. That is, for any sequence of predictions $\vpred = (p_1,\ldots,p_T)\in [0,1]^T$ and any sequence of states $\vstate = (\theta_1,\ldots,\theta_T)\in \{0,1\}^T$,
\begin{equation*}
    \swapV (\vpred,\vstate) \leq \robustcal(\vpred,\vstate)\leq 2\swapV(\vpred,\vstate).
\end{equation*}
\end{theorem}

\begin{proof}
Since $\vbreg_\mu$ is a special case of a Bregman divergence $\breg_S$ for a scoring rule $S$, comparing the definition of $\swapV$ in \eqref{eq:swapV} and the definition of $\robustcal$ in \eqref{eq:cdl}, we immediately get $\swapV (\vpred,\vstate)\leq \robustcal(\vpred,\vstate)$.

Let $Q = \{q_1,\ldots,q_m\}\subseteq [0,1]$ be a finite prediction space such that $p_t\in Q$ for every $t\in [T]$. For $i = 1,\ldots,m$, define $n_i$ and $\hat q_i$ as in \eqref{eq:ni-intro} and \eqref{eq:hat-q-intro}.

Let us consider any proper scoring rule $S:[0,1]\times \{0,1\}\to [0,1]$. By \Cref{thm:score-u}, there exists a convex function $u:[0,1]\to [0,1]$ with bounded sub-gradients $\nabla u(p)\in [-1,1]$ such that \eqref{eq:score-u} holds. We can construct a sequence of twice-differentiable convex functions $u_1,u_2,\ldots$ on the domain $[0,1]$ to approximate $u$ in the following manner. For every $q\in \{0,1,q_1,\ldots,q_m,\hat q_1,\ldots,\hat q_m\}$, the derivative $u_j'(q)$ satisfies $u_j'(q) = \nabla u(q)$ for every $j$, and the function value $u_j$ itself satisfies $u_j(q)\to u(q)$ as $j\to +\infty$. 
Let $S_j$ denote the proper scoring rule generated by $u_j$ as in \eqref{eq:score-u}.
 For every $i = 1,\ldots,m$, we have
\[
\lim_{j\to \infty} \breg_{S_j}(q_i,\hat q_i) = \breg_S(q_i,\hat q_i).
\]

By \Cref{lm:v-decompose}, for every $j = 1,2,\ldots,$
\begin{align*}
\frac 1T\sum_{i\in [m]}n_i\breg_{S_j}(q_i, \empq_i) & = \frac 1T\sum_{i\in [m]}n_i\int_{\kink = 0}^{1} \util''_j(\kink)\cdot\max(1-\kink, \kink)\vbreg_{\kink}(q_i,\hat q_i)\diff \kink\\
& = \int_{\kink = 0}^{1} \util''_j(\kink)\cdot\max(1-\kink, \kink)\frac 1T\sum_{i\in [m]}n_i\vbreg_{\kink}(q_i,\hat q_i)\diff \kink\\
& \le \int_{\kink = 0}^{1} \util''_j(\kink)\cdot\max(1-\kink, \kink)\swapV (\vpred,\vstate)\diff \kink\\
& = \swapV (\vpred,\vstate) \int_{\kink = 0}^{1} \util''_j(\kink)\cdot\max(1-\kink, \kink)\diff \kink\\
& \le \swapV (\vpred,\vstate) \int_{\kink = 0}^{1} \util''_j(\kink)\cdot\diff \kink\\
& = \swapV (\vpred,\vstate)(u_j'(1) - u_j'(0))\\
& = \swapV (\vpred,\vstate)(\nabla u(1) - \nabla u(0))\\
& \le 2\swapV (\vpred,\vstate).
\end{align*}
Therefore, by \Cref{prop:swap-bregman},
\begin{align*}
\cfdl_S(\vpred,\vstate) = \frac 1T\sum_{i=1}^m n_i\breg_{S}(q_i, \empq_i) & = \lim_{j\to \infty}\frac 1T\sum_{i=1}^m n_i\breg_{S_j}(q_i, \empq_i)\\
 & \le 2\swapV (\vpred,\vstate).
\end{align*}
Taking supremum over $S$ proves $\robustcal(\vpred,\vstate) \le 2\swapV(\vpred,\vstate)$.
\end{proof}

\subsection{Computation of $\robustcal$}
\label{sec:computation}

We allow the state space to be non-binary in this section. By solving a linear program, the $\robustcal$ can be computed in time polynomial in the size of the prediction space $|Q|$ and the state space $|\statesp|$.

\begin{theorem}\label{thm:computation}
   Given a sequence of predictions $\vpred$ and states $\vstate$,  suppose $Q$ is the space of predictions. $\robustcal$ can be computed in time polynomial in $|Q|$ and $|\statesp|$.
\end{theorem}

\begin{proof}
    The computation of $\robustcal$ is an optimization problem over the space of proper scoring rules. We follow the idea in \citet{li2022optimization, kleinberg2023u}. The optimal scoring rule can be computed by solving a linear program. Let $\hat{Q} = \{\empq_i\}_i$ be the set of empirical distributions for each $\q_i\in Q$.  Define $\hat{\statesp} = \{\Tilde{\state} \mid j\in |\statesp|, \Tilde{\state}_j = 1, \Tilde{\state}_{j'\neq j} = 0\}$ as the set of indicator predictions of a certain state. Define the space of predictions as $\mathcal{Q} = Q\cup\hat{Q}\cup\hat{\statesp}$. We set the scores $s_{\q, \state}, \forall \q\in \mathcal{Q}, \state\in \statesp$ as variables in the linear program. 
    \begin{align*}
        \max_{\score}\qquad &\frac{1}{T}\sum_{\q\in Q, \state\in \statesp} (s_{\empq, \state} - s_{\q, \state})\empq(\state)&\\ 
        \text{s.t.}\qquad& s_{\q, \state}\in [0, 1], \quad \forall \q\in \mathcal{Q} , \state\in \statesp & \text{(bounded payoff)}\\
        & \sum_{\state}\q(\state)s_{\q, \state}\geq \sum_{\state}\q(\state)s_{\q', \state}, \quad \forall \q, \q'\in \mathcal{Q} & \text{(properness)}
    \end{align*}

    The following proper scoring rule achieves the worst case $\cfdl$.
    \begin{equation*}
        \score(\pred, \state) = s_{\q, \state}, \text{ where }\q = \argmax_{\q'\in \mathcal{Q}}\expect{\state\sim\pred}{s_{\q', \state}}.\qedhere
    \end{equation*}
\end{proof}

\section{$\robustcal$ and Calibration Errors}
\label{sec:connection}

In this section, we discuss connections of our $\robustcal$ to calibration errors in the literature.  We show that both $\ece$ and $\lcal{2}$ are polynomially related to $\robustcal$ (\Cref{thm:k-msr}), but neither is a constant-factor approximation (\Cref{example:k1 k2 msr}). We give examples where the smooth calibration error differs significantly from $\robustcal$ in either direction (\Cref{example:smooth msr}). We also show the U-calibration error lowerbounds $\robustcal$ (\Cref{prop:ucal-msr}), but is not polynomially related (\Cref{example:u calib}).


\begin{theorem}\label{thm:k-msr}
Given samples of predictions and corresponding states, 
\begin{align*}
        \ece^2\leq &\robustcal\leq 2\ece,\\
        \lcal{2}\leq &\robustcal\leq 2\sqrt{\lcal{2}}.
\end{align*}
\end{theorem}

\begin{proof}[Proof of \Cref{thm:k-msr}]
On the upper bound side, by \Cref{claim:kleinberg k1 swap bound}, any proper scoring rule $\score$ with range bounded in $[0,1]$ satisfies $\cfdl_\score\leq 2\ece$, implying $\robustcal\leq 2\ece$.  It is easy to check that $\ece\leq \sqrt{\lcal{2}}$ \citep[see e.g.][]{kleinberg2023u, rothbook}. Thus we get the other upper bound $\robustcal \leq2\sqrt{\lcal{2}}$.

On the lower bound side, by \Cref{lem:k2-quadratic}, we have $\robustcal  \geq \lcal{2} $. Combining this with the fact $\ece\leq \sqrt{\lcal{2}}$ yields $\robustcal\geq \ece^2$.
%
%
%
\end{proof}

All four inequalities in \Cref{thm:k-msr} are tight up to constant factors. We demonstrate the tight examples in \Cref{example:k1 k2 msr}.

\begin{example}\label{example:k1 k2 msr}
    Consider the following two miscalibrated predictors.
    \begin{itemize}
        \item[(a)] (Tight example for upper bounds of $\robustcal$) The state is deterministically $1$. The predictor deterministically predicts $1-\epsilon$. 

        In this case, $\ece = \epsilon$, $\lcal{2} = \epsilon^2$, $\robustcal =\Theta(\epsilon)$.
        \item[(b)](Tight example for lower bounds of $\robustcal$) The $T$ rounds are divided into $\sqrt{T}$ periods, each with $\sqrt{T}$ rounds. In each period $i$, the empirical distribution of the state is $\frac{i}{\sqrt{T}}$, and the predictor predicts $\frac{i}{\sqrt{T}} +  \frac{1}{\sqrt{T}}$.

        In this case, $\ece = \frac{1}{\sqrt{T}}$, $\lcal{2} = \frac{1}{T}$,  $\robustcal \in [\frac{1}{T}, \frac{8}{T}]$.
    \end{itemize}
\end{example}

From \Cref{example:k1 k2 msr}, we see that when predictions concentrate in a small interval, $\ece$ calculates the $\robustcal$ in the correct order as in (a). However, if predictions have high variance as in (b), the calibration error does not simply add up to the total loss in decision. Consider the V-Bregman divergence which corresponds to a decision problem with two actions. Suppose the kink, also the decision threshold, is at $\sfrac{1}{2}$. Miscalibration at extreme predictions near $0$ or $1$ will not induce a $\cfdl$ to the agent. This intuition  is explained by \Cref{lem:swap regret bound on v bregman with kink}, with which we prove the $\asympO(\frac{1}{\sqrt{T}})$ $\cfdl$ later. We state it here and prove it in \Cref{sec:attribute}.

Define the bias in bucket $i$ (i.e.\ conditional $\ece$ calibration error on prediction $\q_i$):
\begin{equation*}
    \G_i = \qcount_i|\q_i - \empq_i|.
\end{equation*}

\begin{restatable*}{lemma}{lemSwapG}\label{lem:swap regret bound on v bregman with kink}
Let $T,m$ be positive integers. Define $Q = \{q_1,\ldots,q_m\}\subseteq [0,1]$ where $q_i = i/m$ for every $i = 1,\ldots,m$.
Given a sequence of predictions $\vpred = (p_1,\ldots,p_T)\in Q^T$ and realized states $\vstate = (\theta_1,\ldots,\theta_T)\in \{0,1\}^T$, define $n_i$ and $\hat q_i$ as in \eqref{eq:ni-intro} and \eqref{eq:hat-q-intro}. Define $G_i:= n_i|\hat q_i - q_i|$. 
Fix a V-Bregman divergence with kink $\kink$, the $\cfdl$ is bounded by
    \begin{equation*}
        \cfdl_\kink(\vpred, \vstate) \leq \frac{2}{T}\sum_{i\in [m]}\left(\G_i - \qcount_i|\q_i - \kink| \right)_+.
    \end{equation*}
\end{restatable*}

\begin{proof}[Proof of \Cref{example:k1 k2 msr}]
  The calculation of $\ece$ and $\lcal{2}$ are straightforward. We show the calculation of MSR separately for (a) and (b).
  \begin{itemize}
      \item [(a)] By \Cref{lem:swap regret bound on v bregman with kink}, $\robustcal\leq 2\swapV\leq 4\epsilon $. We can find a Bregman divergence such that $\cfdl \geq \frac{\epsilon }{2} $. Consider the scoring rule $\score$ that has a V-Bregman divergence with kink $\kink = 1-\frac{\epsilon}{2}$.
      \begin{equation*}
          \cfdl = \frac{|1 - \kink|}{\kink} \geq \frac{\epsilon }{2}.
      \end{equation*}
      Thus, $\robustcal = \Theta(\epsilon)$.
      \item [(b)] We can prove $\swapV = \Theta(1)$. Fix any V-Bregman divergence with kink $\kink$, consider the corresponding  $\cfdl_\kink$.
      \begin{align*}
          \cfdl_\kink &= \frac{1}{T}\cdot\sqrt{T}\sum_{i = 1}^{\sqrt{T}}\frac{|\frac{i}{\sqrt{T}}-\kink|}{\max\{\kink, 1-\kink\}}\left(\ind{\frac{i}{\sqrt{T}}>\kink>\frac{i+1}{\sqrt{T}}} + \ind{\frac{i}{\sqrt{T}}<\kink<\frac{i+1}{\sqrt{T}}}\right)
      \end{align*}
      We notice that for predictions that induces a non-zero $\cfdl$, it must be $|\frac{i}{\sqrt{T}} - \kink|\leq \frac{1}{\sqrt{T}}$. Since $\max\{\kink, 1-\kink\}\geq \frac{1}{2}$,
      \begin{equation*}
          \cfdl_\kink\leq \frac{1}{\sqrt{T}}\cdot 2\cdot \frac{1}{\sfrac{1}{2}}\frac{1}{\sqrt{T}} = \frac{4}{T}.
      \end{equation*}
      We know $\robustcal\leq 2\swapV\leq \frac{8}{T}$.
  \end{itemize}
\end{proof}

$\robustcal$  is not polynomially related to $\smooth$ or $\distcal$. Specifically, example (b) in \Cref{example:smooth msr} shows the  $\asympO(\sqrt{T})$  $\distcal$ guarantee in online calibration does not apply to $\robustcal$.
\begin{example}\label{example:smooth msr}
    We give two examples of miscalibrated predictors. 
    \begin{itemize}
        \item[(a)] (Large $\distcal$, small $\robustcal$) The same example as (b) in \Cref{example:k1 k2 msr}.
        
        $\distcal\geq\smooth\geq \frac{1}{\sqrt{T}}$, $\robustcal\in [\frac{1}{T}, \frac{2}{T}]$.
        \item[(b)] (Small $\distcal$, large $\robustcal$) At the first $\frac{T}{2}$ rounds, $\state = 1$ deterministically, and the predictor predicts $\sfrac{1}{2} + \epsilon$. At the later $\frac{T}{2}$ rounds, $\state = 0$ deterministically, and the predictor predicts $\sfrac{1}{2} - \epsilon$.

        $\smooth\leq\distcal \leq \epsilon $, $\robustcal = \Omega(1)$. 
        
        We can take $\epsilon = \frac{1}{\sqrt{T}}$, which can be arbitrarily small.

    \end{itemize}
\end{example}

\begin{proof}[Proof of \Cref{example:smooth msr}]
    We calculate $\distcal$ and $\robustcal$ separately for each example.
    \begin{itemize}
        \item [(a)] It only remains to show $\smooth \geq \frac{1}{\sqrt{T}}$. By \Cref{def:smooth calibration error}, take Lipschitz function $\sigma(\cdot) = 1$.
        \begin{equation*}
            \smooth\geq \frac{1}{T}\sum_t (\pred_t - \state_t) = \frac{1}{T}\cdot T\cdot\frac{1}{\sqrt{T}} = \frac{1}{\sqrt{T}}.
        \end{equation*}
        \item [(b)] For $\distcal$, a calibrated predictor always predicts $\sfrac{1}{2}$. 
        \begin{equation*}
            \distcal\leq \frac{1}{T}\sum_{t\in [T]}|\pred_t - \frac{1}{2}| = \epsilon.
        \end{equation*}

        For $\robustcal$, consider the V-Bregman divergence with kink $\sfrac{1}{2} + 2\epsilon$. The $\cfdl$ for this V-Bregman divergence is
        \begin{equation*}
            \cfdl_{\sfrac{1}{2}} \geq  \frac{1}{2}(1-\sfrac{1}{2}-2\epsilon) = \sfrac{1}{4}-\epsilon.
        \end{equation*}
        $\robustcal\geq \cfdl_{\sfrac{1}{2}}$ implies $\robustcal = \Omega(1)$.
    \end{itemize}
\end{proof}

Vanishing U-calibration error is necessary but not sufficient for calibration. By definition, the U-calibration error lowerbounds $\robustcal$.

\begin{proposition}\label{prop:ucal-msr}
    For any sequence of predictions and states, 
    \begin{equation*}
        \ucal\leq \robustcal.
    \end{equation*}
\end{proposition}

\cite{kleinberg2023u} gives an example where in the limit as $T\to \infty$, the U-calibration error is $0$, while $\ece$, $\lcal{2}$ and $\smooth$ are non-zero. We present a simpler example showing $\ucal = 0$ is insufficient for calibration and $\robustcal = \Omega(1)$. 

\begin{example}
\label{example:u calib}
    The empirical distribution of the state is $\Pr[\state = 1] = \sfrac{1}{2}$, i.e.\ $\frac{T}{2}$ samples are $1$, $\frac{T}{2}$ samples are $0$. The predictor predicts $\frac{3}{4}$ when the state is $1$, and $\frac{1}{4}$ when the state is $0$.

    In this example, $\ucal = 0$, $\robustcal = \Omega(1)$. 
\end{example}

\begin{proof}[Proof of \Cref{example:u calib}]

First, $\ucal$ is always non-negative. By definition, there exists a degenerate
decision task with a constant payoff, where the decision maker has $0$ external regret. By the same V-shaped decomposition in \Cref{lm:v-decompose} and in \citet{kleinberg2023u}, it only remains to show this predictor is weakly better than predicting $\frac{1}{2}$ on all V-shaped scoring rules in \Cref{eq:V-shaped score}.

Fix each V-shaped scoring rule with kink $\kink\in[0, \frac{1}{4}]$, always predicting $\sfrac{1}{2}$ and the miscalibrated predictor aboth achieve the same payoff. The external regret is thus $0$. For V-shaped scoring rules with kink $\kink\in (\frac{1}{4}, \frac{1}{2}]$, predicting $\frac{1}{2}$ yields payoff $\frac{3}{4}-\frac{1}{4}\cdot\frac{\kink}{1-\kink}$. However, the miscalibrated predictor obtains higher payoff $\frac{3}{4}+\frac{1}{4}\cdot\frac{\kink}{1-\kink}$, which achieves negative external regret. The case for $\kink>\frac{1}{2}$ is similar to the two cases above. We can conclude that the external regret for each downstream decision task is non-positive. 

This example, however, is very miscalibrated. Specifically, consider the V-shaped scoring rule with kink at $\frac{1}{4}-\epsilon$. $\cfdl$ of the miscalibrated predictor is $\Omega(1)$ for predicting $\frac{1}{4}$, where the conditional empirical frequency is $0$. 

\end{proof}

\section{Minimizing Calibration Decision Loss}
\label{sec:main-sqrt T}

In this section, we present our online binary prediction algorithm (\Cref{alg:min MSR}) and prove that it achieves the following low $\robustcal$ guarantee:

\begin{theorem}\label{thm:sqrt MSR}
For $T \ge 2$,
    \Cref{alg:min MSR} runs in time polynomial in $T$ and makes predictions $\vpred = (p_1,\ldots,p_T)$ satisfying
    \begin{equation*}
        \expect{}{\robustcal(\vpred, \vstate)}\leq  O(\frac{\log T}{\sqrt{T}}).
    \end{equation*}
   Here, $\vstate = (\theta_1,\ldots,\theta_T)$ is the sequence of realized states chosen by any adversary in the online binary prediction setting (see \Cref{sec:online}), and the expectation is over the randomness of the algorithm. 
\end{theorem}

We design \Cref{alg:min MSR} such that it makes predictions in a finite set $Q:= \{q_1,\ldots,q_m\}\subseteq [0,1]$, where $q_i = i/m$ for each $i = 1,\ldots,m$. Later we will pick the optimal choice of $m \approx \sqrt T/\log T$. We view each $q_i$ as a bucket, so the prediction $p_t$ made by \Cref{alg:min MSR} in each round $t$ falls into one of the $m$ buckets $q_1,\ldots,q_m$. We use $n_i$ to denote the number of predictions in bucket $i$ (see \eqref{eq:ni-intro}), and use $\hat q_i$ to denote the average value of the realized states corresponding to the $n_i$ predictions (see \eqref{eq:hat-q-intro}).
We define $G_i:= n_i|q_i - \hat q_i|$ as the \emph{bias} from bucket $i$.

In \Cref{sec:attribute}, we prove a key technical lemma which allows us to attribute the $\robustcal$ to the bias $G_i$ from each bucket. We then present \Cref{alg:min MSR} in \Cref{sec:alg} and complete the proof of \Cref{thm:sqrt MSR}.

\subsection{Attributing $\robustcal$ to Bucket-wise Biases}

\label{sec:attribute}

We establish our key technical lemma that allows us to upper bound $\robustcal$ using the bucket-wise biases.
\begin{lemma}[Formal and generalized version of \Cref{lm:attribute-informal}]
\label{lm:attribute}
Let $T,m \ge 2$ be positive integers. Define $Q = \{q_1,\ldots,q_m\}\subseteq [0,1]$ where $q_i = i/m$ for every $i = 1,\ldots,m$.
Given a sequence of predictions $\vpred = (p_1,\ldots,p_T)\in Q^T$ and realized states $\vstate = (\theta_1,\ldots,\theta_T)\in \{0,1\}^T$, define $n_i$ and $\hat q_i$ as in \eqref{eq:ni-intro} and \eqref{eq:hat-q-intro}. Define $G_i:= n_i|\hat q_i - q_i|$. For $\alpha,\beta \ge 0$, define maximum deviation $\LL$:
\begin{equation}
\label{eq:LL}
\LL(\vpred,\vstate) := \max_{1 \le i \le m}\{[G_i - \alpha\sqrt n_i - \beta n_i]_+\}, \quad \text{where $[x]_+:= \max(x,0)$.}
\end{equation}
Then
\[
\robustcal (\vpred, \vstate) \le \frac{4m}{T}\LL(\vpred,\vstate) + \frac{4\alpha}{\sqrt{T}} + 4\beta  + O\Big(\frac{\alpha^2 m \log m}{T}\Big).
\]
\end{lemma}




Our proof of \Cref{lm:attribute} relies on the following helper lemma which controls the $\cfdl$ w.r.t.\ a single V-shaped Bregman divergence.

\lemSwapG


\begin{proof}
\begin{align*}
            \cfdl_\kink(\vpred, \vstate) &=\frac{1}{T}\sum_{i\in [m]}\qcount_i\vbreg_\kink(\q_i, \empq_i)\\
            &=\frac{1}{T}\sum_{i\in [m]}\qcount_i\frac{|\empq_i-\kink|}{\max\{1-\kink, \kink\}}\left(\ind{\q_i< \kink<\empq_i} + \ind{\q_i> \kink>\empq_i}\right) \tag{by \eqref{eq:vbreg}}\\
            &\leq \frac{1}{T}\sum_{i\in [m]}\qcount_i\frac{\left(|\empq_i-\q_i| - |\kink - \q_i|\right)_+}{\max\{1-\kink, \kink\}}\\
            &\leq \frac{2}{T} \sum_{i\in [m]}\left(\G_i - \qcount_i|\q_i - \kink|\right)_+.\qedhere
\end{align*}
\end{proof}

 \begin{proof}[Proof of \Cref{lm:attribute}]
By \Cref{thm:v approx msr}, it suffices to prove 
\begin{equation}
\label{eq:vswap-1}
\swapV(\vpred,\vstate) \le \frac{2m}{T}\LL(\vpred,\vstate) + \frac{2\alpha}{\sqrt{T}} + 2\beta  + O(\frac{\alpha^2 m \log m}{T}).
\end{equation}

For any $\mu\in [0,1]$, by \Cref{lem:swap regret bound on v bregman with kink},
\begin{align}
\cfdl_\mu(\vpred,\vstate) & \le \frac{2}{T}\sum_{i=1}^m(G_i - n_i|q_i - \mu|)_+\notag\\
& \le \frac{2}{T}\sum_{i=1}^m(\LL(\vpred,\vstate) + \alpha\sqrt n_i + \beta n_i - n_i|q_i - \mu|)_+ \notag\\
& \le \frac{2}{T}\sum_{i=1}^m(\LL(\vpred,\vstate) + \beta n_i + (\alpha \sqrt{n_i} - n_i|q_i - \mu|)_+)\tag{because $\LL(\vpred,\vstate) \ge 0$ and $\beta \ge 0$}\\
& = \frac{2m}{T}\LL(\vpred,\vstate) + 2\beta  + \frac{2}{T}\sum_{i=1}^m(\alpha \sqrt {n_i} - n_i|q_i - \mu|)_+.\label{eq:swap-mu-1}
\end{align}
Let us re-arrange $q_1,\ldots,q_m$ in non-decreasing order of $|q_i - \mu|$. That is, we choose a  bijection $\tau$ from $\{1,\ldots,m\}$ to itself such that $|q_{\tau(i)} - \mu|$ is a non-decreasing function of $i$. When $i = 1$, we use the following trivial upper bound:
\begin{equation}
\label{eq:swap-mu-2}
(\alpha \sqrt {n_{\tau(1)}} - n_{\tau(1)}|q_{\tau(1)} - \mu|)_+ \le \alpha \sqrt {n_{\tau(1)}} \le \alpha \sqrt T.
\end{equation}
When $i > 1$, we have $|q_{\tau(i)} - \mu| \ge \Omega(i/m)$, and thus
\begin{equation}
\label{eq:swap-mu-3}
(\alpha \sqrt {n_{\tau(i)}} - n_{\tau(i)}|q_{\tau(i)} - \mu|)_+ \le \frac{\alpha^2}{4|q_{\tau(i)} - \mu|} = O(\alpha^2 m/i).
\end{equation}

Plugging \eqref{eq:swap-mu-2} and \eqref{eq:swap-mu-3} into \eqref{eq:swap-mu-1}, we get
\begin{align*}
\cfdl_\mu(\vpred,\vstate) & \le \frac{2m}{T}\LL(\vpred,\vstate) + 2\beta  + \frac{2\alpha}{ \sqrt T} + O\left(\frac{1}{T}\alpha^2 m \sum_{i=2}^m \frac 1{i}\right)\\
& \le \frac{2m}{T}\LL(\vpred,\vstate) + \frac{2\alpha}{\sqrt{T}} + 2\beta  + O(\frac{\alpha^2 m \log m}{T}).
\end{align*}
This implies \eqref{eq:vswap-1}, as desired.
\end{proof}

\subsection{Efficient $\robustcal$ Minimization Algorithm}
\label{sec:alg}
Given \Cref{lm:attribute}, we can establish the low $\robustcal$ guarantee in \Cref{thm:sqrt MSR} by designing an algorithm (\Cref{alg:min MSR}) that minimizes $\LL(\vpred,\vstate)$.
Specifically, for parameters $m \approx \sqrt T/\log T, \beta = 1/m, \alpha \approx \sqrt {\log T}$, we show that \Cref{alg:min MSR} achieves $\expect{}{\LL(\vpred,\vstate)} = O(\log T)$ (\Cref{lem:bound LL algorithm 1}). 
Our design of \Cref{alg:min MSR} largely follows the ideas from \citet{noarov2023highdimensional}, but we make small but important refinements to obtain a stronger guarantee as needed to prove \Cref{thm:sqrt MSR} (see \Cref{remark:refine}).

In \Cref{alg:min MSR}, we partition the interval $[0,1]$ into $m$ sub-intervals $I_1,\ldots,I_m$ where 
\begin{equation}
\label{eq:intervals}
I_1 = [0,1/m], I_2 = (1/m, 2/m], \ldots, I_m = ((m-1)/m,1]. 
\end{equation}
In each round $t = 1,\ldots,T$, \Cref{alg:min MSR} first computes a prediction $\tilde p_t\in [0,1]$ and then outputs a discretized prediction $p_t$ via rounding. Specifically, $Q = \{q_1,\ldots,q_m\}$ is the discretized prediction space, where $q_i = i/m$ for every $i = 1, \ldots,m$. The prediction $\tilde p_t$ belongs to an interval $I_i$ for a unique index $i = 1,\ldots,m$, and the corresponding discretized prediction is $p_t = q_i\in Q$. We use $n_i$ to denote the number of rounds $t$ in which $p_t = q_i$, or equivalently, $\tilde p_t \in I_i$:
\[
n_i:= \sum_{t=1}^T\ind{p_t = q_i} = \sum_{t=1}^T\ind{\tilde p_t\in I_i}.
\]
For each $i = 1,\ldots,m$, and $\sigma= \pm 1$, we define
\begin{equation}
\label{eq:expert-gain}
l_{i,\sigma}(p,\theta) := \sigma \ind{p\in I_i}(p - \theta) \quad \text{for every prediction $p\in [0,1]$ and state $\theta\in \{0,1\}$.}
\end{equation}

\Cref{alg:min MSR} calls an expert regret minimization oracle $\msmwc$ from \cite{msmwc}.
Here we imagine $2m + 1$ \emph{experts}: one expert for each pair $(i,\sigma)\in [m]\times\{\pm 1\}$ and one extra auxiliary expert. In each round $t$, the oracle  $\msmwc$ computes a distribution over the experts represented by values $w_{t,i,\sigma}\ge 0$, where for each pair $(i,\sigma)$, the value $w_{t,i,\sigma} \ge 0$ is the probability mass on the expert corresponding to $(i,\sigma)$. Thus, the probability mass on the auxiliary expert is $1 - \sum_{(i,\sigma)}w_{t,i,\sigma} \ge 0$.

The distribution computed by $\msmwc$ in each round $t$ is based on the \emph{gains} $l_{t',i,\sigma}\in [-1,1]$ received by each expert $(i,\sigma)$ at Step 7 in each previous round $t' < t$. We set the gain of the auxiliary expert to always be zero. The work of \citet{msmwc} shows a construction of the oracle with the following property:

\begin{lemma}[\cite{msmwc}, applied to \Cref{alg:min MSR}]\label{lemma:msmwc}
For some absolute constant $C>0$, there exists an expert regret minimization oracle $\msmwc$ for step 1 of \Cref{alg:min MSR} with the following properties. Assume $T,m\ge 2$. In each round $t = 1,\ldots,T$, the oracle computes $w_{t,i,\sigma}\ge 0$ for every $i\in [m]$ and $\sigma = \pm 1$ in time $\mathsf{poly}(m)$ such that 
\begin{align}
\sum_{i\in [m],\sigma = \pm 1}w_{t,i,\sigma} & \le 1, \notag\\
\label{eq:0}
-\sum_{t=1}^T\sum_{i\in [m], \sigma = \pm 1}{\w_{t, i, \sigma}}\expert_{t, i, \sigma} & \le C\log(mT),
\end{align}
and for every $i = 1,\ldots,m$ and $\sigma = \pm 1$,
    \begin{equation}
    \label{eq:expert-regret}
       \sum_{t=1}^T\expert_{t, i, \sigma}-\sum_{t = 1}^T\sum_{i'\in [m], \sigma' = \pm 1}{\w_{t, i', \sigma'}}\expert_{t, i', \sigma'}\leq C\left(\log (mT) + \sqrt{\qcount_i\log(mT)}\right).
    \end{equation}
\end{lemma}
For each expert $(i,\sigma)$, the guarantee \eqref{eq:expert-regret} is stronger than more standard guarantees for the experts problem in that the right hand side of \eqref{eq:expert-regret} has a $\sqrt{n_i}$ dependence rather than a $\sqrt T$ dependence. For the auxiliary expert, we get the guarantee \eqref{eq:0}, which can be viewed as a special form of \eqref{eq:expert-regret} with $l_{t,i,\sigma} = 0$ and $n_i = 0$.

\begin{algorithm}[thbp]
    \caption{Algorithm for $\robustcal$ minimization. }
    \label{alg:min MSR}
    \begin{algorithmic}
    \State Parameters: positive integers $m, T$; $\epsilon > 0$; discretized prediction space $Q = \{q_1,\ldots,q_m\}$ where $q_i = i/m$; intervals $I_1,\ldots,I_m$ partitioning $[0,1]$ as defined in \eqref{eq:intervals}; functions $l_{i,\sigma}$ as defined in \eqref{eq:expert-gain}.
    \For{each round $t= 1, \ldots, T$}
        
        
        \State 1.\ Compute expert weights $w_{t,i,\sigma}$ for every $i\in [m]$ and $\sigma = \pm 1$ using the expert regret minimization oracle $\msmwc$ from \Cref{lemma:msmwc}.
        \State 2. For any distribution $s$ over $[0,1]$, define
        \begin{equation}
        \label{eq:ht}
        h_t(s):= \max_{\theta\in \{0,1\}}\expect{p \sim s}{\sum_{i\in [m],\sigma = \pm 1}\w_{t, i, \sigma}l_{i,\sigma}(p, \theta)}.
        \end{equation}

        \State 3. Find distribution $s_t$ such that $h_t(s_t) \le  \epsilon$.
        \State 4. Draw $\prednoround_t\in [0,1]$ from distribution $\s_t$.
        \State 5.\ \textbf{Output} $p_t := q_i$, where $i$ is the unique index in $\{1,\ldots,m\}$ satisfying $\tilde p_t \in I_i$.
        \State 6. Receive the realized state $\state_t\in \{0,1\}$. 
        \State 7. Calculate expert gains $\expert_{t, i, \sigma}$ for every $i\in[m]$ and $\sigma = \pm 1$:
                \begin{equation}\label{eq:expert loss}
  \expert_{t, i, \sigma} := l_{i,\sigma}(\tilde p_t,\theta_t).
\end{equation}
    \EndFor
    \end{algorithmic}
\end{algorithm}

The following lemma shows that Step 3 of \Cref{alg:min MSR} can be computed efficiently:

\begin{lemma}[\cite{noarov2023highdimensional}]\label{lemma:low minmax value}
Let $\epsilon > 0$ be a parameter.
At step 3 of \Cref{alg:min MSR}, a solution $s_t$ satisfying $h_t(s_t)\le \epsilon$ always exists and can be computed in time $\mathsf{poly}(\epsilon^{-1})$.
\end{lemma}

The existence of $s_t$ in \Cref{lemma:low minmax value} can be proved using the minimax theorem. 
We refer the reader to \citet{noarov2023highdimensional} for a complete proof of (a more general version of) \Cref{lemma:low minmax value} which includes an efficient algorithm for computing $s_t$ using the Follow-the-Perturbed-Leader approach.

We are now ready to prove that \Cref{alg:min MSR} guarantees a small value of $\LL(\vpred,\vstate)$ in expectation:

\begin{lemma}\label{lem:bound LL algorithm 1}
Let $C > 0$ be the absolute constant from \Cref{lemma:msmwc}.
Assume $m,T \ge 2$ and $\varepsilon = 1/T$ in \Cref{alg:min MSR}.
Let $\vpred = (p_1,\ldots,p_T)$ be the predictions made by \Cref{alg:min MSR} on the adversarially chosen states $\vstate = (\theta_1,\ldots,\theta_T)$.
Define $\LL$ as in \eqref{eq:LL} for $\alpha = C\sqrt{\log (mT)}, \beta = 1/m$. Then,
    \begin{equation*}
        \expect{}{\LL(\vpred, \vstate)} \leq O(\log (mT)),
    \end{equation*}
where the expectation is over the randomness of \Cref{alg:min MSR}.
\end{lemma}

\begin{remark}
\label{remark:refine}
We define $\LL(\vpred,\vstate)$ in \eqref{eq:LL} as a maximum over the buckets $i$. Thus, \Cref{lem:bound LL algorithm 1} shows an upper bound on the ``expectation of maximum'', which is stronger than guarantee stated in \citet{noarov2023highdimensional} on the ``maximum of expectation''. In our proof below, we use a slightly more careful analysis than what is used by \citet{noarov2023highdimensional} to achieve the stronger guarantee.
\end{remark}

\begin{proof}[Proof of \Cref{lem:bound LL algorithm 1}]

For every $i = 1,\ldots,m$, by the definition of $G_i$ in \Cref{lm:attribute}, we have
\begin{align*}
G_i = n_i |q_i - \hat q_i| & = 
\left|\sum_{t=1}^T \ind{p_t = q_i}(p_t - \theta_t)\right|\\
 & \le \left|\sum_{t=1}^T \ind{p_t = q_i}(\tilde p_t - \theta_t)\right| + n_i/m\\
 & = \max_{\sigma = \pm 1}\sum_{t=1}^Tl_{i,\sigma}(\tilde p_t, \theta_t) + n_i/m \tag{by \eqref{eq:expert-gain}}\\
 & = \max_{\sigma = \pm 1}\sum_{t=1}^Tl_{t,i,\sigma} + n_i/m.\tag{by \eqref{eq:expert loss}}
\end{align*}
Therefore, by \eqref{eq:expert-regret},
\begin{align*}
G_i - n_i/m - \alpha\sqrt {n_i} \le \left( \max_{\sigma = \pm 1} \sum_{t = 1}^T l_{t,i,\sigma}\right) - \alpha \sqrt {n_i} \le \sum_{t = 1}^T\sum_{i'\in [m], \sigma' = \pm 1}{\w_{t, i', \sigma'}}\expert_{t, i', \sigma'} + C \log(mT).
\end{align*}
By \eqref{eq:0},
\[
0 \le \sum_{t = 1}^T\sum_{i\in [m], \sigma = \pm 1}{\w_{t, i, \sigma}}\expert_{t, i, \sigma} + C \log(mT).
\]
Combining the two inequalities above, we get
\begin{equation}
\label{eq:LL-1}
\LL(\vpred,\vstate) \le \sum_{t = 1}^T\sum_{i\in [m], \sigma = \pm 1}{\w_{t, i, \sigma}}\expert_{t, i, \sigma} + C \log(mT).
\end{equation}
By the definition of $h_t$ in \eqref{eq:ht}, the guarantee of $h_t(s_t) \le \epsilon$, and the fact that $\tilde p_t$ is drawn from $s_t$, we get
\begin{equation}
\label{eq:expLL-1}
    \expect{}{\sum_{t = 1}^T\sum_{i\in [m], \sigma = \pm 1}{\w_{t, i, \sigma}}\expert_{t, i, \sigma}} \le \sum_{t=1}^T \epsilon \le 1.
\end{equation}
Combining \eqref{eq:LL-1} and \eqref{eq:expLL-1}, we get
    \begin{equation*}
       \expect{}{ \LL(\Tilde{\vpred}, \vstate)}\leq O(\log (mT)).\qedhere
    \end{equation*}
\end{proof}
We now complete the proof of our main theorem.
\begin{proof}[Proof of \Cref{thm:sqrt MSR}]
We choose $\epsilon = 1/T$ in \Cref{alg:min MSR} and set $m = \Theta(\sqrt T/\log T)$.
Following the setting of \Cref{lem:bound LL algorithm 1}, we define $\LL$ as in \eqref{eq:LL} for $\alpha = C\sqrt{\log (mT)}, \beta = 1/m$.

We have
     \begin{align*}
       \expect{}{\robustcal(\vpred,\vstate)}&\leq \frac{4m}{T}\expect{}{\LL(\vpred,\vstate)} + \frac{4\alpha}{ \sqrt T }+ 4\beta  + O\Big(\frac{\alpha^2 m \log m}{T}\Big) \tag{by \Cref{lm:attribute}}\\
 & \le O\Big(\frac{m \log (mT)}{T}\Big) +\frac{4\alpha}{ \sqrt T} + 4\beta  + O\Big(\frac{\alpha^2 m \log m}{T}\Big) \tag{by \Cref{lem:bound LL algorithm 1}}\\
 & = O\Big(\frac{m \log T}{T}\Big) + O\Big(\frac{\sqrt {\log T}}{\sqrt{T}}\Big) + O\Big(\frac{\log T}{\sqrt{T}}\Big) + O\Big(\frac{\log T}{\sqrt{T}}\Big)\\
 & =O\Big(\frac{\log T}{\sqrt{T}}\Big).
     \end{align*}
The running time guarantee of \Cref{alg:min MSR} follows from \Cref{lemma:msmwc,lemma:low minmax value}.
\end{proof}




\bibliographystyle{econ}
\bibliography{ref}
\appendix

\section{Minimax Proof for Minimizing $\robustcal$}
\label{appdx:minimax}

We prove the existence of an algorithm that achieves $O(\frac{\log T}{\sqrt{T}})$ $\robustcal$ via the minimax theorem. The minimax theorem allows us to assume that the adversary's (randomized) strategy is fixed and known by the predictor. Our proof demonstrates a remarkable difference between $\ece$ and $\robustcal$ - simply by  truthfully reporting the mean of each state (conditioned on the history and rounded to a suitable finite subset of $[0,1]$), our predictor achieves the $O(\frac{\log T}{\sqrt{T}})$ CDL rate that nearly matches the natural $\Omega(\frac 1{\sqrt T})$ lower bound.

In this setup, a predictor $F$ makes a prediction $p_t\in [0,1]$ at each time step $t = 1,2,\ldots$, and an adversary $A$ picks an outcome $\theta_t\in \{0,1\}$. Both $p_t$ and $\theta_t$ are chosen based on the history $h_{t-1} = (p_1,\theta_1,p_2,\theta_2,\ldots,p_{t-1},\theta_{t-1})$. The goal of the predictor is to minimize $\robustcal (h_T)$.

We can identify a predictor $F$ by its strategy at each time step. That is, we can write $F = (F_1,\ldots,F_T)$, where $F_t$ is a function mapping from $h_{t - 1}$ to $p_t$. Similarly, we can identify the adversary $A$ by its strategy $A_t$ at each time step $t$, writing $A$ as $(A_1,\ldots,A_T)$. Given the strategies $F$ and $A$ of both players, we use $h_{F,A}$ to denote the history $h_T = (p_1,\theta_1,\ldots,p_T,\theta_T)$ generated by executing these strategies.

We will allow the predictor to randomize and play a mixed strategy which can be identified as a probability distribution $\cF$ over all predictor strategies $F$.
\begin{theorem}[Existence]
\label{thm:minimax}
There exists a mixed predictor strategy $\cF$ such that
\[
\max_{A}\E_{F\sim \cF}[\robustcal(h_{F,A})] = O\Big(\frac{\log T}{\sqrt{T}}\Big),
\]
where the maximum is over all strategies $A$ of the adversary.
\end{theorem}

In fact, we will prove a stronger version of \Cref{thm:minimax}, where we restrict the predictor to make predictions $p_t$ in a discretized space $Q=\{q_i=\frac{i}{m}\}_{i\in[m]}\subseteq [0,1]$. 
We say a predictor strategy $F = (F_1,\ldots,F_T)$ is \emph{discretized} if each $F_t$ is a function mapping from history $h_{t-1}$ to $p_t\in Q$. This discretization makes the (pure) strategy space of the predictor to be finite, allowing us to apply the minimax theorem.

We prove the following statement which implies \Cref{thm:minimax} immediately: there exists a positive integer $m$ such that
\begin{equation*}
\min_{\cF}\max_A\E_{F\sim \cF}[\robustcal(h_{F,A})] = O\Big(\frac{\log T}{\sqrt{T}}\Big),
\end{equation*}
where the minimum is over a distribution $\cF$ over discretized predictor strategies $F$, and the maximum is over all strategies $A$ of the adversary.

By the minimax theorem,
\[
\min_{\cF}\max_A\E_{F\sim \cF}[\robustcal(h_{F,A})] = \max_{\cA}\min_{F}\E_{A\sim \cA}[\robustcal(h_{F,A})].
\]
Thus,  it suffices to reverse the order of play and  prove \Cref{lem:minimax}.

\begin{lemma}
\label{lem:minimax}
There exists a positive integer $m$ such that
    \begin{equation*}
        \max_{\cA}\min_{F}\E_{A\sim \cA}[\robustcal(h_{F,A})] \le  O\Big(\frac{\log T}{\sqrt{T}}\Big).
    \end{equation*}
    where the minimum is  over a deterministic discretized predictor strategy $F$, and the maximum is over a distribution over strategies $A$ of the adversary.
\end{lemma}

To show \Cref{lem:minimax}, consider the truthful strategy of the predictor. At round $t$,  let $\prednoround_t$ denote the conditional expectation of the adversary's choice $\theta_t$ given past history. The predictor outputs the closest value $\pred_t$ in $Q$ to $\prednoround_t$.

\begin{lemma}\label{lem:martingale-bound}
Let $\vpred= (p_1,\ldots,p_T)$ and $\vstate = (\theta_1,\ldots,\theta_T)$ be the sequence of predictions and states generated when the predictor uses the truthful strategy above.
For $\alpha \ge 0$ and $\beta = \frac{1}{m}$, define $\LL(\vpred,\vstate)$ as in \Cref{eq:LL}. We have 
\begin{equation*}
    \expect{}{\LL(\vpred,\vstate)}\leq 2T^2 m \exp{(-\alpha^2/2)}.
\end{equation*}
\end{lemma}


\begin{proof}[Proof of \Cref{lem:martingale-bound}]

For bucket $i$, construct random variables $X_1\sps i, \dots X_T\sps i$ such that $X_j\sps i = \state_{t_j}- \tilde p_{t_j}$, where $t_j$ is the index of the $j$-th round in which the predictor predicts $q_i$. If the number $n_i$ of rounds with prediction $q_i$ is smaller than $j$, we define $X_j\sps i = 0$.
%
 For any $n\in [T]$, by Azuma's inequality,
\begin{equation*}
    \Pr\left[\left|\sum_{j = 1}^{n} X_j\sps i\right|> \alpha\sqrt{n}\right]\leq 2 \exp{(\frac{-\alpha^2}{2})}.
\end{equation*}

By the union bound,
\begin{equation}
    \Pr\left[\exists n\in [T], i\in [m], \left|\sum_{j = 1}^{n} X_j\sps i\right| > \alpha\sqrt{n}\right]\leq 2Tm \exp{(\frac{-\alpha^2}{2})}.
\end{equation}
Therefore, with probability at least $1 - 2Tm \exp{(\frac{-\alpha^2}{2})}$, for every $i\in [m]$, we have
\[
\left|\sum_{j=1}^{n_i}X_j\sps i\right| \le \alpha \sqrt{n_i}.
\]
This implies
\[
G_i = \left|\sum_{j=1}^{n_i} (p_{t_j} - \theta_{t_j})\right| \le \left|\sum_{j=1}^{n_i} (\tilde p_{t_j} - \theta_{t_j})\right| + n_i/m \le \alpha \sqrt {n_i} + \beta n_i.
\]
%
%
%
%
Thus, $\LL(\vpred,\vstate) = 0$ with probability at least $1-2Tm \exp{(\frac{-\alpha^2}{2})}$.
Since $\LL(\vpred,\vstate)\leq T$, 
\begin{equation*}
    \expect{}{\LL(\vpred,\vstate)}\leq 2T^2 m\exp{(\frac{-\alpha^2}{2})}.
\end{equation*}

\end{proof}

\begin{proof}[Proof of \Cref{lem:minimax}]

Taking $\alpha = 2\sqrt{\log(Tm)}$ in \Cref{lem:martingale-bound}, we get
\begin{equation*}
    \expect{}{\LL(\vpred,\vstate)}\leq 2.
\end{equation*}
We choose $m = \Theta(\sqrt{T}/\log(T))$ buckets and set $\beta = \frac{1}{m}$. By \Cref{lm:attribute}, 

\begin{align*}
    \robustcal (\vpred, \vstate) &\le \frac{4m}{T}\LL(\vpred,\vstate) + \frac{4\alpha}{\sqrt{T}} + 4\beta  + O\Big(\frac{\alpha^2 m \log m}{T}\Big).\\
    &\le O\Big(\frac{1}{\sqrt{T}\log T}\Big) +  O\Big(\frac{\sqrt{\log T}}{\sqrt{T}}\Big) + O\Big(\frac{\log T}{\sqrt{T}}\Big) + O\Big(\frac{\log T}{\sqrt{T}}\Big)\\
    & = O\Big(\frac{\log T}{\sqrt{T}}\Big).\qedhere
\end{align*}
\end{proof}


\end{document}